\documentclass[journal]{IEEEtran}
\ifCLASSINFOpdf
  \usepackage[pdftex]{graphicx}
\else
  \usepackage[dvips]{graphicx}
\fi
\usepackage{amsmath}
\usepackage{amsfonts}
\usepackage{amsthm}
\usepackage{bbm}

\usepackage{algorithm}
\usepackage{algorithmicx}
\usepackage{algpseudocode}

\ifCLASSOPTIONcompsoc
  \usepackage[caption=false,font=normalsize,labelfont=sf,textfont=sf]{subfig}
\else
  \usepackage[caption=false,font=footnotesize]{subfig}
\fi

\usepackage{stfloats}

\usepackage{cite}
\usepackage{multirow}

\newtheorem{lemma}{Lemma}
\newtheorem{proposition}{Proposition}
\newtheorem{theorem}{Theorem}

\newcommand{\specialcell}[2][c]{\begin{tabular}[c]{@{}#1@{}}#2\end{tabular}}

\begin{document}
%
\title{Efficient Spiking Neural Networks with Logarithmic Temporal Coding}
%
%
%

\author{Ming~Zhang,
        Nenggan~Zheng\textsuperscript{*},~\IEEEmembership{Member,~IEEE},
        De~Ma,
        Gang~Pan,~\IEEEmembership{Member,~IEEE},
        and~Zonghua~Gu,~\IEEEmembership{Senior Member,~IEEE,}
\thanks{Ming Zhang, De Ma, and Gang Pan are with College of Computer Science, Zhejiang University, Hangzhou, China.}
\thanks{Nenggan Zheng is with Qiushi Academy for Advanced Studies, Zhejiang University, Hangzhou, China (e-mail: zng@cs.zju.edu.cn).}
\thanks{Zonghua Gu is with Department of Applied Physics and Electronics, Ume\r{a} University, 90187 Ume\r{a}, Sweden.}
\thanks{*Corresponding author.}}

%
%

\ifCLASSOPTIONpeerreview
\markboth{Journal of \LaTeX\ Class Files}%
{Efficient Spiking Neural Networks with Logarithmic Temporal Coding}
\else
\markboth{Journal of \LaTeX\ Class Files}%
{Zhang \MakeLowercase{\textit{et al.}}: Efficient Spiking Neural Networks with Logarithmic Temporal Coding}
\fi
%



\maketitle

\begin{abstract}
A Spiking Neural Network (SNN) can be trained indirectly by first training an Artificial Neural Network (ANN) with the conventional backpropagation algorithm, then converting it into an SNN. The conventional rate-coding method for SNNs uses the number of spikes to encode magnitude of an activation value, and may be computationally inefficient due to the large number of spikes. Temporal-coding is typically more efficient by leveraging the timing of spikes to encode information. In this paper, we present Logarithmic Temporal Coding (LTC), where the number of spikes used to encode an activation value grows logarithmically with the activation value; and the accompanying Exponentiate-and-Fire (EF) spiking neuron model, which only involves efficient bit-shift and addition operations. Moreover, we improve the training process of ANN to compensate for approximation errors due to LTC. Experimental results indicate that the resulting SNN achieves competitive performance at significantly lower computational cost than related work.
\end{abstract}



%
\IEEEpeerreviewmaketitle

\section{Introduction}

\IEEEPARstart{D}{eep} Learning based on Artificial Neural Networks (ANNs) has achieved tremendous success in many application domains in recent years \cite{Hu2017, Kowsari2017, Cazenave2018}. Spiking Neural Networks (SNNs) use neuron action potentials, or spikes, for event-driven computation and communication. If the number of spikes are low, then most neurons and synapses in an SNN may be idle most of the time, hence the hardware implementation of SNNs can be much more efficient than conventional ANNs used in Deep Learning for inference tasks. Training, or learning, algorithms for SNNs \cite{gutig2006tempotron, Ponulak2010, Florian2012, mohemmed2012span, Neftci13, Diehl15STDP, Bohte2002, McKennoch2006, Booij2005, GhoshDastidar2009, Xu2013} are an active area of research, and are not as mature as conventional Deep Learning.
Several recent SNN learning algorithms based on spiking variants of backpropagation \cite{OConnor16, Lee16} achieved good performance, but their neuron models incur high computational cost. One alternative is to use ANN-to-SNN conversion techniques \cite{Cao15, Diehl15Conversion, Rueckauer2017, Hunsberger15, Liu17, Zambrano17, Rueckauer18}, which works by first training an ANN with the conventional backpropagation algorithm, then converting it into an SNN. Most existing ANN-to-SNN conversion methods are based on rate-coding,
where activations in the ANN are approximated by firing rates of the corresponding spike trains in the SNN, and the number of spikes for encoding a real-valued activation grows linearly with the activation value. For current methods \cite{Diehl15Conversion, Rueckauer2017} to achieve performance comparable to the ANN, the neurons in the SNN have to fire a large number of spikes, which leads to high computational cost. Although several recent methods \cite{Zambrano17, Rueckauer18} reduced the number of spikes by employing more efficient neural coding,
these methods relied on complex neuron models that continually perform expensive operations.

In this paper, we propose an ANN-to-SNN conversion method based on novel \emph{Logarithmic Temporal Coding (LTC)},
where the number of spikes for encoding an activation grows logarithmically with the activation value in the worse case.
LTC is integrated with the \emph{Exponentiate-and-Fire (EF)} spiking neuron model. \emph{Note that the EF neuron model is not biologically realistic. It is an artificial model that we designed to use in conjunction with LTC for efficient computation in an SNN.}
If implemented with fixed-point arithmetic, an EF neuron performs a bit shift every time step and an addition for every incoming spike.
Furthermore, we introduce approximation errors of LTC into the ANN, and leverage the training process of the ANN to compensate for the approximation errors,
eliminating most of performance drop due to ANN-to-SNN conversion.
Compared with rate-coding methods, our temporal-coding method achieves similar performance at significantly lower computational cost.
Experimental results show that, for a CNN architecture with sufficient model capacity,
the proposed method outperforms rate-based coding, achieving test accuracy of 99.41\% on the MNIST dataset, and computational cost reduction of 93.61\%.

%

\section{Related work}


Learning for single-layer SNNs is a well-studied topic.
Supervised learning algorithms aimed to train an SNN to classify input spatiotemporal patterns \cite{gutig2006tempotron}
or to generate control signals with precise spike times in response to input spatiotemporal patterns \cite{Ponulak2010, Florian2012, mohemmed2012span}.
The Tempotron rule \cite{gutig2006tempotron} trained a spiking neuron to perform binary classification by firing one or more spikes in response to its associated class.
ReSuMe \cite{Ponulak2010} trained spiking neurons to generate target spike trains in response to given spatiotemporal patterns.
Supervised learning was achieved by combining learning windows of Hebbian rules and a concept of remote supervision.
The E-learning rule of Chronotron \cite{Florian2012} improved memory capacity by minimizing a modified version of the Victor and Purpura (VP) distance between the output spike train and the target spike train with gradient descent.
SPAN \cite{mohemmed2012span} also achieved improved memory capacity over ReSuMe,
but with a simpler learning rule than Chronotron.
The learning rule was a spiking variant of the Delta rule,
where the input, output, and target values are replaced with
convolved spike trains.
These algorithms depend on predefined target spike trains,
which are not available for neurons in the hidden layers of a multi-layer SNN.
Unsupervised learning rules aimed to train an SNN to detect spatiotemporal patterns or extract features from input stimuli.
In \cite{Diehl15STDP}, a population of spiking neurons connected with lateral inhibitory synapses were trained using Spike Time-Dependent Plasticity (STDP) to recognize different spatiotemporal patterns.
In \cite{Neftci13}, an event-driven variation of contrastive divergence was proposed to train a restricted Boltzmann machine constructed with integrate-and-fire neurons.
These algorithms rely on specific network topologies with a single layer of spiking neurons.
All the learning algorithms are limited to SNNs with a single layer of neurons.
There is a large performance gap between the resulting SNNs and traditional ANNs.

Multi-layer SNNs are more difficult to train than single-layer SNNs.
Backpropagation \cite{LeCun98} cannot be directly applied to multi-layer SNNs due to the discontinuity associated with spiking activities.
SpikeProp \cite{Bohte2002} adapted backpropagation for SNN,
circumventing the discontinuity by assuming the membrane potential to be a linear function of time for a small region around spike times.
SpikeProp was extended by later works to use Resilient Propagation and QuickProp \cite{McKennoch2006},
and to train neurons in hidden layers and the output layer to fire multiple spikes \cite{Booij2005, GhoshDastidar2009, Xu2013}.
However, there is still a large performance gap between these SNNs and traditional ANNs.
Recent works avoided making assumptions about the discontinuity.
In \cite{OConnor16}, a custom spiking neuron model incorporated a spike generation algorithm to approximate intermediate values of both the forward pass and the backward pass with spike trains.
The spike generation algorithm had to add the encoded value to its internal state for every neuron at every time step.
In \cite{Lee16}, the membrane potential of a neuron was assumed to be a differentiable function of postsynaptic potentials and the afterpotential,
and the backward pass propagated errors through the postsynaptic potentials and the afterpotential instead of input and output spike times.
The exponential decay of postsynaptic potentials and afterpotentials require two multiplications be performed for every neuron at every time step.
These learning algorithms trained small SNNs with several layers to achieve comparable performance to that of traditional ANNs.
However, they rely on complex neuron models that perform expensive arithmetic operations every time step.
Furthermore, how these algorithms scale to deeper SNNs remains unclear.

Another line of work trained SNNs indirectly by converting a trained ANN into its equivalent SNN.
In \cite{Cao15}, an ANN with Rectified Linear Unit (ReLU) nonlinearity was trained using backpropagation
and the weights were then directly mapped to an SNN of Integrate-and-Fire (IF) neurons with the same topology.
In a similar way,
an ANN with Softplus \cite{Hunsberger15} or Noisy Softplus \cite{Liu17} nonlinearity
could be converted to an SNN of more biologically plausible Leaky Integrate-and-Fire (LIF) neurons.
There was a significant performance gap between the resulting SNN and the original ANN.
The performance gap was narrowed by weight normalization \cite{Diehl15Conversion} and resetting membrane potential by subtracting the firing threshold \cite{Rueckauer2017}.
With these improvements, the resulting SNNs achieved performance comparable to the corresponding ANNs.
All of these ANN-to-SNN conversion methods were based on rate coding,
where the number of spikes it takes to encode an activation grows linearly with the activation.
Empirically, the neurons in the SNN have to maintain high firing rates to achieve a comparable performance to the original ANN.
Since the computational cost a spiking neuron incurs is proportional to the number of incoming spikes,
spike trains generated according to rate coding impose high computational cost on downstream neurons.

Recent ANN-to-SNN conversion methods reduced the number of spikes used to encode activations
by employing more efficient neural coding.
In \cite{Zambrano17}, an ANN was converted to an Adapting SNN (AdSNN) based on synchronous Pulsed Sigma-Delta coding.
When driven by a strong stimulus,
an Adaptive Spiking Neuron (ASN) adaptively raises its dynamic firing threshold every time it fires a spike,
reducing its firing rate.
However, an ASN has to perform four multiplications every time step to update its postsynaptic current, firing threshold, and refractory response.
In \cite{Rueckauer18}, an ANN was converted to an SNN based on temporal coding,
where an activation in the ANN was approximated by the latency to the first spike of the corresponding spike train in the SNN.
Thus, at most one spike needs to be fired for each activation.
However, each Time-To-First-Spike (TTFS) neuron keeps track of synapses which have ever received an input spike,
and has to add the sum of the synaptic weights to its membrane potential every time step.
Although these methods reduce the number of spikes,
their complex neuron models still incur high computational cost.

ANN-to-SNN conversion approximates real-valued activations with spike trains.
The approximation errors contribute to the performance gap between the SNN and the ANN.
Fortunately, a deep ANN can be trained to tolerate the approximation errors,
if the approximation errors are introduced during the training phase.
In \cite{Miyashita16}, each activation of an ANN was approximated with a power of two,
where the exponents of the powers were constrained within a set of several consecutive integers.
The error tolerance of an ANN allows it to compensate for approximation errors in the corresponding SNN during the training phase,
which in turn helps close the performance gap between the SNN and the ANN.

Different from existing ANN-to-SNN conversion methods,
we reduce both the number of spikes and the complexity of the neuron model.
We propose encoding activations with Logarithmic Temporal Coding (LTC),
where the number of spikes grows logarithmically with the encoded activation in the worst case.
If implemented with fixed-point arithmetic,
our Exponentiate-and-Fire (EF) neuron model involves only bit shifts and additions.
A neuron performs a bit shift every time step and an addition for every incoming spike.

\section{Method}


Every time a spiking neuron receives an input spike, the membrane potential of the neuron is increased by the postsynaptic potential (PSP).
Evaluation of PSPs contribute to most of the computational cost of an SNN.
To reduce the number of spikes used to encode every activation throughout the ANN,
we propose \emph{Logarithmic Temporal Coding (LTC)}.
A real-valued activation is first approximated by retaining a predefined subset of bits in its binary representation.
Then, a spike is generated for each of the remaining 1 bits and no spike is generated for the 0 bits.
The number of spikes for encoding an activation grows logarithmically, rather than linearly, with the activation in the worst case.

We propose \emph{Exponentiate-and-Fire (EF) neuron} used in conjunction with LTC, which performs equivalent computation to that of an analog neuron with Rectified Linear Unit (ReLU) nonlinearity. 
Furthermore, we propose \emph{Error-tolerant ANN training}, which leverages the ANN training process to compensate for approximation errors introduced by LTC and
reduces the chance for EF neurons to fire undesired spikes.

We use the term ``activation'' to refer to output values of all analog neurons in an ANN, including neurons in the input, hidden and output layers.

\subsection{Logarithmic temporal coding}

To encode a real-valued activation into a spike train,
the activation is first represented as a binary number.
Then, the activation is approximated by retaining only a subset of the bits of the binary number at a predefined set of consecutive positions;
the other bits of the binary number are set to zero.
Finally, for each remaining 1 bit of the binary number,
a spike is generated with spike timing determined by the position of the bit in the binary number,
while no spike is generated for the 0 bits.

An real-valued activation $a \geq 0$ can be represented as the sum of a possibly infinite series of powers of two $2^e$ with different integer exponents $e$.
We approximate the real-valued activation $a$ by constraining the exponents $e$ within a predefined exponent range $\{ e_{min}, \ldots , e_{max} \}$,
i.e., a finite set of consecutive integers from $e_{min}$ to $e_{max}$.
This approximation can be formulated as a closed-form equation:
\begin{equation}
    \tilde{a} =
    \begin{cases}
        0 & \text{if } a < 2^{e_{min}}, \\
        \lfloor a / 2^{e_{min}} \rfloor \cdot 2^{e_{min}} & 2^{e_{min}} \leq a < 2^{e_{max} + 1}, \\
        2^{e_{max} + 1} - 2^{e_{min}} & \text{if } a \geq 2^{e_{max} + 1}.
    \end{cases}
    \label{eq_logarithmic_approximation_multi_power}
\end{equation}
Since the approximation defined by Eqn. \ref{eq_logarithmic_approximation_multi_power} may involve multiple powers of two,
we refer to this approximation as \emph{Multi-Power Logarithmic Approximation (Multi-Power LA)}.

As a special case, if we further require the approximation to involve at most one single power of two,
the approximation $\tilde{a}$ reduces to \emph{Single-Power Logarithmic Approximation (Single-Power LA)}:
\begin{equation}
    \tilde{a} =
    \begin{cases}
        0 & \text{if } a < 2^{e_{min}}, \\
        2^{\lfloor \log_2{a} \rfloor} & 2^{e_{min}} \leq a < 2^{e_{max} + 1}, \\
        2^{e_{max}} & \text{if } a \geq 2^{e_{max} + 1}.
    \end{cases}
    \label{eq_logarithmic_approximation_single_power}
\end{equation}
We refer to multi-power LA and single-power LA collectively as \emph{Logarithmic Approximation (LA)}.


In order to generate an LTC spike train from a logarithmic approximation $\tilde{a}$,
we define a time window with $T = e_{max} - e_{min} + 1$ discrete time steps $\{ 0, 1, \dots , T - 1 \}$.
If a power of two $2^e$ contributes to the logarithmic approximation $\tilde{a}$,
i.e., $2^e$ is present in the series of powers of two of $\tilde{a}$,
then a spike is present in the LTC spike train with a spike time $t = e_{max} - e$.
There are two variants of LTC:
\emph{Multi-spike LTC} corresponds to multi-power LA,
while \emph{Single-spike LTC} corresponds to single-power LA.

Obviously, single-spike LTC encodes a real-valued activation into a spike train with at most one single spike.
For multi-spike LTC, we derive an upper bound of the number of spikes used to encode a real-valued activation,
as Proposition \ref{prp_number_of_spikes_upper_bound_multi_spike_ltc} states.



\begin{proposition}
Suppose multi-spike LTC encodes a real value $a$ into a spike train with $n_s$ spikes.
If $a < 2^{e_{min}}$, then $n_s = 0$;
if $2^{e_{min}} \leq a < 2^{e_{max} + 1}$, then $n_s \leq \lfloor \log_2{a} \rfloor - e_{min} + 1$;
if $a \geq 2^{e_{max} + 1}$, then $n_s = e_{max} - e_{min} + 1$.
\label{prp_number_of_spikes_upper_bound_multi_spike_ltc}
\end{proposition}

\begin{proof}
Let $\tilde{a}$ be the multi-power LA of $a$.
Any power $2^e$ with an integer exponent $e \geq \lfloor \log_2{a} \rfloor + 1$ cannot contribute to $\tilde{a}$,
because $2^e > a \geq \tilde{a}$.
For a power $2^e$ with an integer exponent $e$ to contribute to $\tilde{a}$,
$e \in (- \infty , \lfloor \log_2{a} \rfloor ] \cap \{ e_{min}, \ldots , e_{max} \}$.

If $a < 2^{e_{min}}$, $(- \infty , \lfloor \log_2{a} \rfloor ] \cap \{ e_{min}, \ldots , e_{max} \} = \emptyset$.
Hence, no power of two contributes to the multi-power LA of $a$.
According to LTC, the spike train contains no spike, hence $n_s = 0$.

If $2^{e_{min}} \leq a < 2^{e_{max} + 1}$,
$(- \infty , \lfloor \log_2{a} \rfloor ] \cap \{ e_{min}, \ldots , e_{max} \} = \{ e_{min}, \ldots , \lfloor \log_2{a} \rfloor \}$.
In the worst-case, every $2^e$ with an integer exponent in the set $\{ e_{min}, \ldots , \lfloor \log_2{a} \rfloor \}$ contributes to $\tilde{a}$.
Thus, $n_s \leq \lfloor \log_2{a} \rfloor - e_{min} + 1$.

If $a \geq 2^{e_{max} + 1}$,
$(- \infty , \lfloor \log_2{a} \rfloor ] \cap \{ e_{min}, \ldots , e_{max} \} = \{ e_{min}, \ldots , e_{max} \}$.
Every $2^e$ with an integer exponent $e \in \{ e_{min}, \ldots , e_{max} \}$ contributes to $\tilde{a}$, hence $n_s = e_{max} - e_{min} + 1$.
\end{proof}

The logarithmic increase in the number of spikes for LTC is much slower than the linear increase for rate coding.
The slow increase comes at the cost of significant approximation error.
Since both LA and LTC are deterministic,
the approximation error can be easily introduced into activations of an ANN during the training phase.
We leverage the training process of an ANN to compensate for the approximation errors,
as detailed in Section \ref{sec_ann_training}.

\subsection{Exponentiate-and-Fire (EF) neuron model}

Figure \ref{fig_ef_neuron_computation_graph} illustrates the Exponentiate-and-Fire (EF) neuron model. An EF neuron integrates input spikes using an exponentially growing PSP kernel,
and generates output spikes using an exponentially growing afterhyperpolarizing potential (AHP) kernel.
With the exponentially growing kernel, an EF neuron is able to perform computation that is equivalent to the computation of an analog neuron with ReLU nonlinearity. 

\begin{figure*}[!t]
    \centering
    \includegraphics[width=0.7 \textwidth]{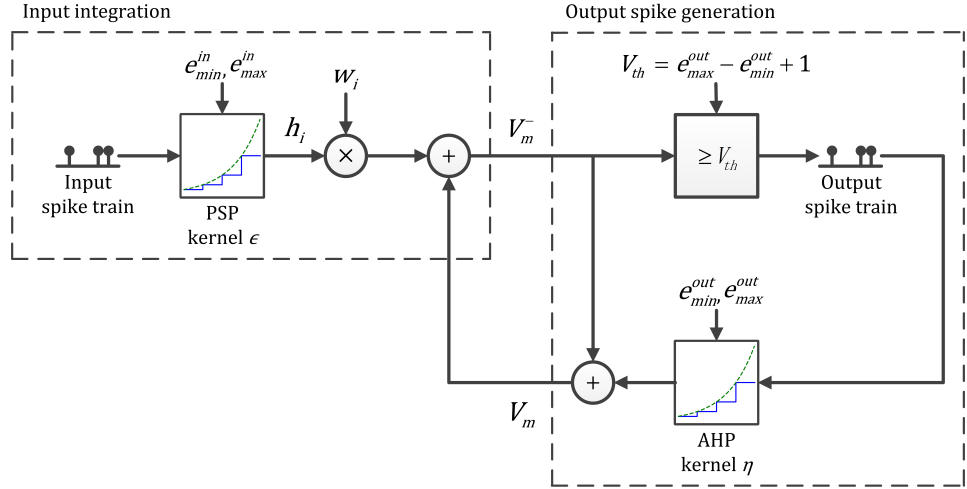}
    \caption{Computation graph of an Exponentiate-and-Fire (EF) neuron.}
    \label{fig_ef_neuron_computation_graph}
\end{figure*}

The EF neuron model is based on the Spike Response Model (SRM) \cite{Gerstner2002} with discrete time. 
The membrane potential $V_m(t)$ at time $t \in \mathbb{Z}$ is given by:
\begin{equation}
    V_m(t) = \sum_{i \in \Gamma}{w_i \cdot h_i(t)} + \sum_{t^{out} \in \mathcal{F}^{out}}{\eta (t - t^{out}) \cdot \mathbbm{1}(t \geq t^{out})}
    \label{eq_membrane_potential}
\end{equation}
where $\Gamma$ is the set of synapses; $w_i$ is weight of synapse $i$; $h_i$ is the total PSP elicited by the input spike train at synapse $i$; $\mathcal{F}^{out} \subseteq \mathbb{Z}$ is the set of output spike times; $\eta (t - t^{out})$ is the AHP elicited by the output spike at time $t^{out}$; $\mathbbm{1}( \cdot )$ evaluates to 1 if and only if the condition enclosed within the brackets is true.
$V_m(t)$ is the \emph{pre-reset membrane potential} immediately before the reset:
\begin{equation}
    V_m^-(t) = V_m(t) - \eta (0) \cdot \mathbbm{1}(t \in \mathcal{F}^{out})
    \label{eq_pre_reset_membrane_potential}
\end{equation}

\subsubsection{Input integration}
Input spike trains of a neuron are generated using the input exponent range $\{ e^{in}_{min}, \ldots , e^{in}_{max} \}$ of the neuron,
and presented to the neuron during its input time window $\{ 0, \ldots , T^{in} - 1 \}$, where $T^{in} = e^{in}_{max} - e^{in}_{min} + 1$. The exponentially growing PSP kernel $\epsilon (t - t^{in})$ is used to integrates input spikes:
\begin{equation}
    \epsilon (t - t^{in}) = 2^{e^{in}_{min}} \cdot 2^{t - t^{in}} \cdot \mathbbm{1}(t \geq t^{in})
    \label{eq_psp_kernel}
\end{equation}
where $t \in \mathbb{Z}$ is the current time; $t^{in} \in \mathbb{Z}$ is time of the input spike. With this PSP kernel, the PSP elicited by an input spike is equal to $2^{e^{in}_{min}}$ at the spike time $t = t^{in}$, and doubles every time step thereafter.

The total PSP elicited by an input spike train at the synapse $i$ is the superposition of PSPs elicited by all spikes in the spike train:
\begin{equation}
    h_i(t) = \sum_{t^{in} \in \mathcal{F}^{in}_i}{\epsilon (t - t^{in})}
    \label{eq_psp}
\end{equation}
where $\mathcal{F}^{in}_i$ is the set of spike times of the input spike train.



If the EF neuron does not fire any output spike before $t = T^{in} - 1$,
then no output spike would interfere with input integration,
and the EF neuron computes a weighted sum of the LAs of its input spike trains,
as Lemma \ref{lem_ef_membrane_potential_recovers_preactivation} states.

\begin{lemma}
The pre-reset membrane potential of an EF neuron $V_m^-(T^{in} - 1) = \sum_{i \in \Gamma }{w_i \cdot \tilde{a}_i}$,
if the EF neuron does not fire any output spike during the time interval $\{ 0, \ldots , T^{in} - 2 \}$,
where $\tilde{a}_i$ is the LA of the $i$-th input LTC spike train.
\label{lem_ef_membrane_potential_recovers_preactivation}
\end{lemma}

\begin{proof}
According to LA, $\tilde{a}_i = \sum_{k}{2^{e^{in}_{i,k}}}$.
According to LTC, the spike time corresponding to the power $2^{e^{in}_{i,k}}$ is $t^{in}_{i,k} = e^{in}_{max} - e^{in}_{i,k}$.
The total PSP elicited by the $i$-th input LTC spike train at $t = T^{in} - 1$ is given by
\begin{equation}
h_i(T^{in} - 1) = \sum_{t^{in}_{i, k} \in \mathcal{F}^{in}_i}{\epsilon (T^{in} - 1 - t^{in}_{i, k})} = \sum_{k}{2^{e^{in}_{i, k}}} = \tilde{a}_i
\end{equation}
Since the EF neuron does not fire any output spike before $t = T^{in} - 1$,
$V_m^-(T^{in} - 1)$ reduces to a weighted sum of PSPs elicited by the input spike trains:
\begin{equation}
V_m^-(T^{in} - 1) = \sum_{i \in \Gamma}{w_i \cdot h_i(T^{in} - 1)} = \sum_{i \in \Gamma}{w_i \cdot \tilde{a}_i}
\end{equation}
completing the proof.
\end{proof}

\subsubsection{Output spike generation}
\label{sec_spike_generation}

The goal of an EF neuron is to generate an output LTC spike train that encodes $\max (V_m^-(T^{in} - 1), 0)$
using its output exponent range $\{ e^{out}_{min}, \ldots , e^{out}_{max} \} \subseteq \mathbb{Z}$ and
present the spike train within its output time window.
The output time window $\{ T^{in} - 1, \ldots , T^{in} + T^{out} - 2 \}$ starts at the last time step $T^{in} - 1$ of the input time window,
and lasts for $T^{out} = e^{out}_{max} - e^{out}_{min} + 1$ time steps.

An EF neuron generates an output spike train by thresholding its exponentially growing membrane potential.
Specifically, the EF neuron doubles its membrane potential every time step after the time step $T^{in} - 1$,
as dictated by the exponentially growing PSP kernel and AHP kernel (detailed below),
until its pre-reset membrane potential reaches its firing threshold $V_{th} = 2^{e^{out}_{max}}$ from below,
when it fires an output spike at time $t^{out}$, and its membrane potential is reset.

A \emph{Multi-Spike EF neuron} resets its membrane potential by subtracting the firing threshold from it:
\begin{equation}
    \eta (t - t^{out}) = - V_{th} \cdot 2^{t - t^{out}}
    \label{eq_reset_by_subtraction}
\end{equation}
A \emph{Single-Spike EF neuron} resets its membrane potential to 0:
\begin{equation}
    \eta (t - t^{out}) = - V_m^-(t^{out}) \cdot 2^{t - t^{out}}
    \label{eq_reset_to_zero}
\end{equation}

After resetting its membrane potential,
a multi-spike EF neuron doubles its membrane potential every time step thereafter,
and may fire subsequent output spikes. In contrast, a single-spike EF neuron does not fire any subsequent spike,
since its membrane potential remains zero after the reset.

If the EF neuron receives all input spikes within its input time window,
then no input spike would interfere with output spike generation during its output time window,
and the EF neuron generates the desired output LTC spike train within its output time window,
as Lemma \ref{lem_ef_neuron_spike_generation_consistency} states.

\begin{lemma}
An EF neuron generates an output LTC spike train that encodes $\max (V_m^-(T^{in} - 1), 0)$
using its output exponent range and
presents the spike train within its output time window,
if the EF neuron does not receive any input spike after the end of its input time window.
\label{lem_ef_neuron_spike_generation_consistency}
\end{lemma}

We prove Lemma \ref{lem_ef_neuron_spike_generation_consistency} in Appendix \ref{sec_consistency_of_ef_neuron_spike_generation_with_ltc}.

\begin{theorem}
An EF neuron performs equivalent computation to the computation of an analog neuron with ReLU nonlinearity,
and encodes the result into its output LTC spike train,
if the following conditions hold:
\begin{enumerate}
\item All input spikes are received within its input time window, and
\item No output spikes are fired before the beginning of its output time window.
\end{enumerate}
\label{the_ef_relu_equivalence}
\end{theorem}

\begin{proof}
Theorem \ref{the_ef_relu_equivalence} follows from Lemmas \ref{lem_ef_membrane_potential_recovers_preactivation} and \ref{lem_ef_neuron_spike_generation_consistency}.
\end{proof}

However, with the spike generation mechanism alone,
an EF neuron may fire undesired output spikes outside its output time window.
An undesired early output spike before the output time window interferes with input integration of the neuron.
In addition,
the output time window of a layer $l$ of EF neurons is the input time window of the next layer $l + 1$.
An undesired late output spike after the output time window interferes with output spike generation of the downstream neurons.
Undesired output spikes break the equivalence between EF neurons and analog ReLU neurons, which in turn degrades the performance of the SNN.

In order to prevent undesired output spikes of an EF neuron from affecting the downstream neurons,
we allow output spikes within the output time window to travel to the downstream neurons,
and discard undesired output spikes outside the output time window.
Furthermore, we reduce the chance for an EF neuron to fire an undesired early output spike
by suppressing excessively large activations of the corresponding analog ReLU neuron,
as detailed in Section \ref{sec_ann_training}.

Algorithm \ref{alg_neuron_dynamics} shows operations an EF neuron performs at every time step.
First, the membrane potential $V_m$ is doubled (Eqn. \ref{eq_psp_kernel}, \ref{eq_reset_by_subtraction} and \ref{eq_reset_to_zero}).
Then, the input current $I$ is calculated by summing up weights $w_i$ of the synapses that receive an input spike at the current time step (Eqn. \ref{eq_membrane_potential}).
The input current is scaled by the resistance $2^{e^{in}_{min}}$ (Eqn. \ref{eq_psp_kernel}) and the result is added to the membrane potential (Eqn. \ref{eq_membrane_potential}).
If the membrane potential is greater than or equal to the firing threshold $V_{th}$,
an output spike is fired,
and the membrane potential is reset accordingly (Eqn. \ref{eq_reset_by_subtraction} and \ref{eq_reset_to_zero}).

From Algorithm \ref{alg_neuron_dynamics}, it can be seen that the EF neuron model can be efficiently implemented in hardware with fixed-point arithmetic.
If $V_m$ is implemented as a fixed-point number, it can be doubled by a bit shift;
if $V_m$ is implemented as a floating-point number, it can be doubled by an addition to its exponent.
The multiplication by $2^{e^{in}_{min}}$ can be avoided by pre-computing $w^i \cdot 2^{e^{in}_{min}}$ for every synaptic weight $w_i$ and
using the scaled synaptic weights at run-time. The other arithmetic operations are additions and subtractions.

\subsection{Error-tolerant ANN training}
\label{sec_ann_training}



Both LTC approximation errors and undesired early output spikes
contribute to the performance gap between an ANN and the corresponding SNN.
We introduce the approximation errors into the activations of the ANN
by applying logarithmic approximation to every non-negative activation,
and rely on the training process to compensate for the approximation errors.
Furthermore, we regularize the loss function with the \emph{Excess Loss} to suppress excessively large activations,
which in turn reduces the chance for an EF neuron to fire an undesired early output spike.

For every analog neuron of the ANN,
we apply LA to its non-negative activations,
so that the downstream neurons receive the approximate activations instead of the original activations.
The variant of LA corresponds to the variant of LTC used to generate the corresponding spike train in the SNN.
Negative pre-activations of the output layer are not approximated using LA and remain unchanged.
For each layer $l$,
the minimum exponent $e^{out, (l)}_{min}$ and the maximum exponent $e^{out, (l)}_{max}$
within the output exponent range are tuned as hyperparameters, similar to \cite{Miyashita16}.
To reduce the number of hyperparameters, we use the same output exponent range for all hidden layers.

As can be seen in Eqn. \ref{eq_logarithmic_approximation_multi_power} and \ref{eq_logarithmic_approximation_single_power},
the derivative of the LA $\tilde{a}$ w.r.t. the real-valued activation $a$ is zero almost everywhere,
which prevents backpropagation from updating parameters of the bottom layers of the ANN.
To allow gradients to pass through LA,
for both variants of LA,
we define the derivative of $\tilde{a}$ w.r.t. $a$ as
\begin{equation}
\frac{\mathrm{d} \tilde{a}}{\mathrm{d} a} =
\begin{cases}
1 & \text{if } a < 2^{e_{max} + 1}, \\
0 & \text{if } a \geq 2^{e_{max} + 1}.
\end{cases}
\label{eq_logarithmic_approximation_derivative}
\end{equation}


In order to suppress excessively large activations,
we define the \emph{Excess Loss} $L_{excess}$ as
\begin{align}
\begin{split}
& L_{excess} = \\
& \sum_m{\sum_l{\sum_j{(\max (a^{(l)}_{m, j} - (2^{e^{out, (l)}_{max} + 1} - 2^{e^{out, (l)}_{min}}), 0))^2}}} / 2
\end{split}
\end{align}
where the outer sum runs across training examples $m$,
the middle sum runs across all layers $l$ of the ANN,
the inner sum runs across all neurons $j$ of the layer $l$,
and $a^{(l)}_{m, j}$ is the activation of the $j$-th neuron of the layer $l$ for the $m$-th training example.
The excess loss punishes large positive activations of every layer $l$ that are greater than $2^{e^{out, (l)}_{max} + 1} - 2^{e^{out, (l)}_{min}}$.

\begin{algorithm}[H]
\caption{Operations performed by an EF neuron at every time step.}
\label{alg_neuron_dynamics}
\algsetblockdefx{InputIntegrationStart}{InputIntegrationEnd}{3}{}{Input integration:}{}
\algsetblockdefx{OutputSpikeGenerationStart}{OutputSpikeGenerationEnd}{9}{}{Output spike generation:}{}
\begin{algorithmic}
    \InputIntegrationStart
        \State $V_m := V_m \times 2$
        \State $I := \sum_{i: \text{synapse }i\text{ receives a spike}}{w_i}$
        \State $V_m := V_m + I \times 2^{e^{in}_{min}}$
    \OutputSpikeGenerationStart
        \If {$V_m \geq V_{th}$}
            \State Fire an output spike
            \If {the neuron is a multi-spike EF neuron}
                \State $V_m := V_m - V_{th}$
            \EndIf
            \If {the neuron is a single-spike EF neuron}
                \State $V_m := 0$
            \EndIf
        \EndIf
\end{algorithmic}
\end{algorithm}


The excess loss $L_{excess}$ is added to the loss function $L$ of the ANN, which is to be minimized by the training process:
\begin{equation}
L = L(x ; \theta) + \lambda L_{excess}
\end{equation}
where $L(x ; \theta)$ is the loss of the ANN on training data $x$ given parameters $\theta$,
and $\lambda > 0$ is a hyperparameter that controls the strength of the excess loss.

Although the excess loss does not completely prevent EF neurons from firing undesired early output spikes,
it makes undesired early output spikes unlikely.
Our experiments show that performance of an SNN with LTC is very close to the performance of the corresponding ANN with LA;
the negative impact of undesired early output spikes seems to be negligible.



\section{Experimental results}


\subsection{Experimental setup}

We conduct our experiments on a PC with an nVidia GeForce GTX 1060 GPU with a 6 GB frame buffer, a quad-core Intel Core i5-7300HQ CPU, and 8 GB main memory.
We use TensorFlow \cite{tensorflow2015} not only for training and testing ANNs,
but also for simulating SNNs.
For each SNN, we build a computation graph with operations performed by the SNN at every time step,
where every spiking neuron outputs either 1 or 0 to indicate whether it fires an output spike or not.
The computation graph is run once for every time step with appropriate input values.

We use the MNIST dataset of handwritten digits \cite{LeCun98}, which consists of 70000 28x28-pixel greyscale images of handwritten digits,
divided into a training set of 60000 images and a test set of 10000 images.
For hyperparameter tuning, we further divide the original training set into a training set of 55000 images and a validation set of 5000 images.
The test set is only used to test ANNs and SNNs after all hyperparameters are fixed.

We use two CNN architectures in our experiments.
One is the \emph{CNN-small} architecture (12C5@28x28-P2-64C5@12x12-P2-F10) with limited model capacity.
This architecture was also used in previous work \cite{Diehl15Conversion, Liu17}.
The other is the \emph{CNN-large} architecture (32C5@28x28-P2-64C5@14x14-P2-F1024-F10) \cite{TensorFlow2018}.

\subsection{Configuration of training and testing}

We consider 5 types of CNNs, each for both CNN-small and CNN-large:
\begin{enumerate}
\item \emph{CNN-original}: Original CNNs with zero biases, ReLU nonlinearity, and average pooling.
CNNs of this type are converted to two types of SNNs.
The \emph{SNN-rate-IF-rst-zero} type uses the reset-to-zero mechanism \cite{Diehl15Conversion},
while the \emph{SNN-rate-IF-rst-subtract} type uses the reset-by-subtraction mechanism \cite{Rueckauer2017}.
We refer to SNN-rate-IF-rst-zero and SNN-rate-IF-rst-subtract collectively as \emph{SNN-rate-IF}.
Since data-based normalization was shown to outperform model-based normalization, the weights of the CNNs are normalized with data-based normalization.
\item \emph{CNN-TF}: Same as CNN-original, except that the transfer function proposed in \cite{Zambrano17} is used as the nonlinearity.
The corresponding SNN type is \emph{SNN-ASN}, where SNNs consist of Adaptive Spiking Neurons (ASNs) \cite{Zambrano17}.
We do not implement the arousal mechanism.
\item \emph{CNN-CR}: Same as CNN-original,
except that clamped ReLU \cite{Rueckauer18} is used as the nonlinearity, and that max-pooling is used instead of average-pooling.
The corresponding SNN type is \emph{SNN-TTFS}, where SNNs consist of Time-To-First-Spike (TTFS) neurons \cite{Rueckauer18}.
\item \emph{CNN-multi-power-LA}: Same as CNN-original, except that all activations throughout the CNN are approximated with multi-power LA.
The corresponding SNN type is \emph{SNN-multi-spike-LTC}, where EF neurons in the hidden and output layers generate multi-spike LTC spike trains.
\item \emph{CNN-single-power-LA}: Same as CNN-multi-power-LA, except that activations of hidden neurons are approximated with single-power LA.
The corresponding SNN type is \emph{SNN-single-spike-LTC}, which is the same as SNN-multi-spike-LTC, except that the EF neurons in the hidden layers generate single-spike LTC spike trains.
\end{enumerate}
We refer to CNN-multi-power-LA and CNN-single-power-LA collectively as CNN-LA,
and SNN-multi-spike-LTC and SNN-single-spike-LTC collectively as SNN-LTC.
For each CNN type, we train five CNNs separately with the same hyperparameters and convert them to SNNs.

For SNN-rate-IF, the maximum input rate for generating an input spike train is 1 spike per time step,
since this maximum input rate was shown to achieve the best performance \cite{Diehl15Conversion}.
For CNN-TF and SNN-ASN, we adopt the hyperparameters for the transfer function and ASNs in \cite{Zambrano17}.
The resting threshold $\theta _0$ and the multiplicative parameter $m_f$ are set to a large value 0.1 to decrease firing rates of ASNs.
For both CNN-LA types, Table \ref{tbl_hyperparameters_cnn_la} shows exponent ranges for different layers and the strength of the excess loss.

\begin{table}[!t]
\renewcommand{\arraystretch}{1.3}
\caption{Exponent ranges and strength of excess loss for CNNs of CNN-LA types.}
\label{tbl_hyperparameters_cnn_la}
\centering
\begin{tabular}{c c c c c}
\hline
\multirow{2}{*}{\specialcell{CNN\\arch.}} & \multicolumn{3}{c}{Exponent ranges} & \multirow{2}{*}{\specialcell{Excess\\loss}} \\
& Input & Hidden & Output & \\
\hline
CNN-small & $\{ -7, \ldots , 0 \}$ & $\{ -3, \ldots , 0 \}$ & $\{ -3, \ldots , 4 \}$ & $0.1$ \\
CNN-large & $\{ -7, \ldots , 0 \}$ & $\{ -7, \ldots , -4 \}$ & $\{ -3, \ldots , 4 \}$ & $0.01$ \\
\hline
\end{tabular}
\end{table}

For SNN-rate-IF and SNN-ASN types, each of the SNNs is simulated for 500 time steps. For SNN-TTFS, simulation for an input image is stopped after the output layer fires the first output spike \cite{Rueckauer18}.

\subsection{Performance evaluation}

Table \ref{tbl_test_accuracies_ann_snn_conversion} compares final average test accuracies of our ANN-to-SNN conversion methods with those of previous ANN-to-SNN conversion methods.
The ``Method'' column shows SNN types, where the ``SNN-'' prefix is omitted.
``small'' and ``large'' in round brackets denote the CNN-small and CNN-large architectures, respectively.
The ``Dev.'' column shows the maximum difference between the test accuracy of an SNN and the test accuracy of the corresponding CNN.
For the SNN-rate-IF types, since input spike trains are generated stochastically,
we test each of these SNNs five times. For each combination of CNN architecture and CNN/SNN type,
the final average test accuracy in the table is obtained by averaging the final test accuracies of all test runs of the neural networks.

\begin{table}[!t]
\renewcommand{\arraystretch}{1.3}
\caption{Comparison of final average test accuracies of ANN-to-SNN methods.}
\label{tbl_test_accuracies_ann_snn_conversion}
\centering
\begin{tabular}{l c c c c}
\hline
\multirow{2}{*}{Method} & \multicolumn{2}{c}{Test accuracy (\%)} & \multirow{2}{*}{\specialcell{Dev.\\(\%)}} & \multirow{2}{*}{\# neurons} \\
& CNN & SNN & & \\
\hline
Rate-IF-rst-zero (small) \cite{Diehl15Conversion} & 99.25 & 99.20 & 0.16 & $1.37 \times 10^4$ \\
\specialcell{Rate-IF-rst-subtract\\(small) \cite{Rueckauer2017}} & 99.25 & 99.25 & 0.06 & $1.37 \times 10^4$ \\
ASN (small) \cite{Zambrano17} & \textbf{99.43} & \textbf{99.43} & 0.04 & $1.37 \times 10^4$ \\
TTFS (small) \cite{Rueckauer18} & 99.22 & 98.53 & 0.83 & $1.37 \times 10^4$ \\
\specialcell{Multi-spike-LTC (small)\\{[this work]}} & 99.23 & 99.23 & \textbf{0.00} & $1.37 \times 10^4$ \\
\specialcell{Single-spike-LTC (small)\\{[this work]}} & 99.03 & 99.03 & \textbf{0.00} & $1.37 \times 10^4$ \\
\hline
Rate-IF-rst-zero (large) \cite{Diehl15Conversion} & 99.27 & 99.24 & 0.09 & $4.80 \times 10^4$ \\
\specialcell{Rate-IF-rst-subtract\\(large) \cite{Rueckauer2017}} & 99.27 & 99.27 & 0.12 & $4.80 \times 10^4$ \\
ASN (large) \cite{Zambrano17} & 99.45 & \textbf{99.44} & 0.04 & $4.80 \times 10^4$ \\
TTFS (large) \cite{Rueckauer18} & \textbf{99.47} & 99.20 & 0.44 & $4.80 \times 10^4$ \\
\specialcell{Multi-spike-LTC (large)\\{[this work]}} & 99.38 & 99.38 & \textbf{0.00} & $4.80 \times 10^4$ \\
\specialcell{Single-spike-LTC (large)\\{[this work]}} & 99.41 & 99.41 & 0.02 & $4.80 \times 10^4$ \\
\hline
Rate-LIF-Softplus \cite{Hunsberger15} & N/A & 98.36 & N/A & 710 \\
Rate-LIF-Noisy-Softplus \cite{Liu17} & 99.05 & 98.85 & 0.20 & $1.37 \times 10^4$ \\
\hline
\end{tabular}
\end{table}

For the CNN-small architecture,
SNN-multi-spike-LTC achieves an average test accuracy that is lower than that of SNN-ASN and similar to those of the SNN-rate-IF types.
SNN-single-spike-LTC achieves a lower average test accuracy than those of SNN-multi-spike-LTC and the SNN-rate-IF types.
Both SNN-LTC types achieve a significantly higher average test accuracy than SNN-TTFS.

The difference in average test accuracy between SNN-rate-IF, SNN-ASN, and SNN-LTC is closely related to the model capacities of the corresponding CNN types.
With a small exponent range size (4 for hidden layers),
multi-power LA significantly decreases the precision of activations by mapping them to a few discrete values.
The decrease in precision leads to a decrease in the model capacity of CNN-multi-power-LA.
Hence multi-power LA can be seen as a regularizer.
Single-power LA is a stronger regularizer than multi-power LA,
since it further decreases the precision for activations.
By contrast, the transfer function of CNN-TF maps real-valued activations to an interval of real numbers,
which allows for much higher precision than the logarithmic approximations.
Hence, the transfer function is a weaker regularizer than the logarithmic approximations.

For a small CNN architecture like CNN-small,
which has limited model capacity even if all activations are real values,
the strong regularization of the logarithmic approximations has a negative effect on the CNN-LA types' ability of modeling training data.
By contrast, the weak regularization of the transfer function has a negligible effect on CNN-TF's ability of modeling training data,
but helps it achieve a higher average test accuracy than CNN-original by mitigating overfitting.

For the CNN-large architecture, which has sufficient model capacity,
both the logarithmic approximations and the transfer function have negligible effect on the CNN types' ability of modeling training data;
they mitigate overfitting and help CNN-TF and the CNN-LA types achieve a higher average test accuracy than CNN-original.
Therefore, the SNN-LTC types outperform the SNN-rate-IF types and achieve similar average test accuracies to that of the SNN-ASN type.

As shown in the ``Dev.'' column of Table \ref{tbl_test_accuracies_ann_snn_conversion},
for the SNN-LTC types, the test accuracy of every SNN is very close to the test accuracy of the corresponding CNN.
The difference in test accuracy is slightly larger for CNN-TF and SNN-ASN,
and much larger for other CNN and SNN types, especially for CNN-CR and SNN-TTFS.
For CNN-large, the performance gap between SNN-TTFS and CNN-CR prevents SNN-TTFS from achieving a higher average test accuracy than the SNN-LTC types,
although CNN-CR achieves a higher average test accuracy than the CNN-LA types.
There seems to be a closer similarity in behavior between SNN-LTC and CNN-LA than between other SNN types and their corresponding CNN types.
The close similarity between SNN-LTC and CNN-LA in turn suggests that
the excess loss is very effective in preventing EF neurons from firing undesired early spikes;
the impact of few undesired early spikes is negligible.

For both CNN-large and CNN-small, the SNN-LTC types outperform SNN types based on LIF neurons.

\subsection{Computational cost evaluation}

In this section, we compare the computational cost of our ANN-to-SNN conversion method with related work \cite{Diehl15Conversion, Rueckauer2017, Zambrano17, Rueckauer18}.

In an SNN, every time a spike arrives at a synapse, which is referred to as a synaptic event,
a postsynaptic potential is added to the membrane potential of the postsynaptic neuron.
These operations contribute to most of the computational cost of an SNN.
We use the average number of synaptic events that an SNN processes for every input image as a metric for the computational cost of the SNN.
In addition, we also count the average number of spikes fired by all neurons of an SNN for every input image.

Figure \ref{fig_computational_cost_accuracy_cnn_small} shows the experimental results for CNN-small.
For each of SNN-rate-reset-zero, SNN-rate-reset-subtract, SNN-ASN, and SNN-TTFS,
the computational cost and test accuracy at every time step during a test run of an SNN are plotted as a point.
For every time step, these computational costs and test accuracies are averaged over all test runs of the SNN type.
The resulting average computational costs and average test accuracies are plotted as a line.
For SNN-LTC types, only the final computational cost and the final test accuracy are shown for every SNN.

\begin{figure*}[!t]
    \centering
    \subfloat[]{
        \includegraphics[width=\columnwidth]{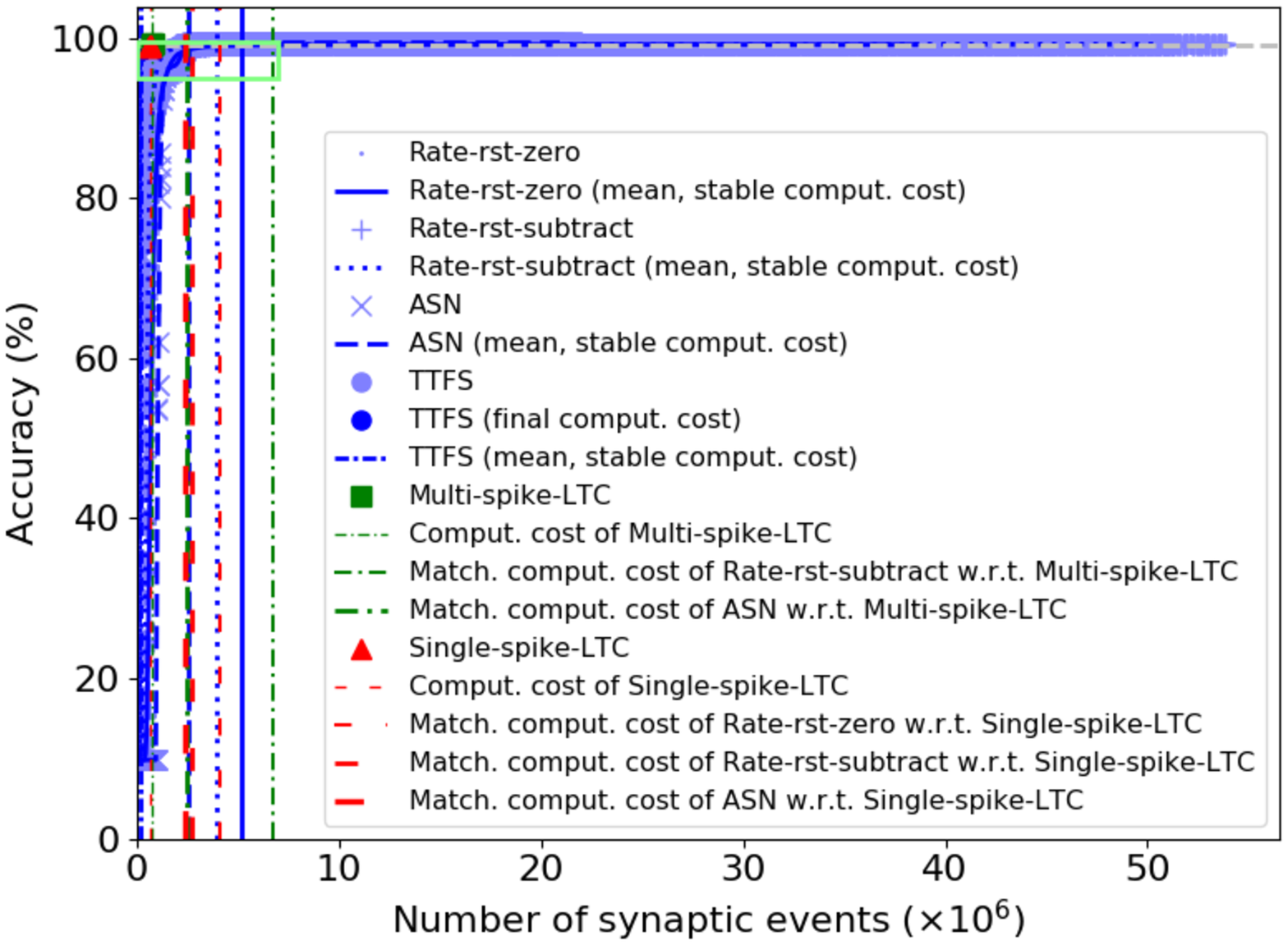}
        \label{fig_computational_cost_accuracy_cnn_small_synaptic_events_complete}}
    \subfloat[]{
        \includegraphics[width=\columnwidth]{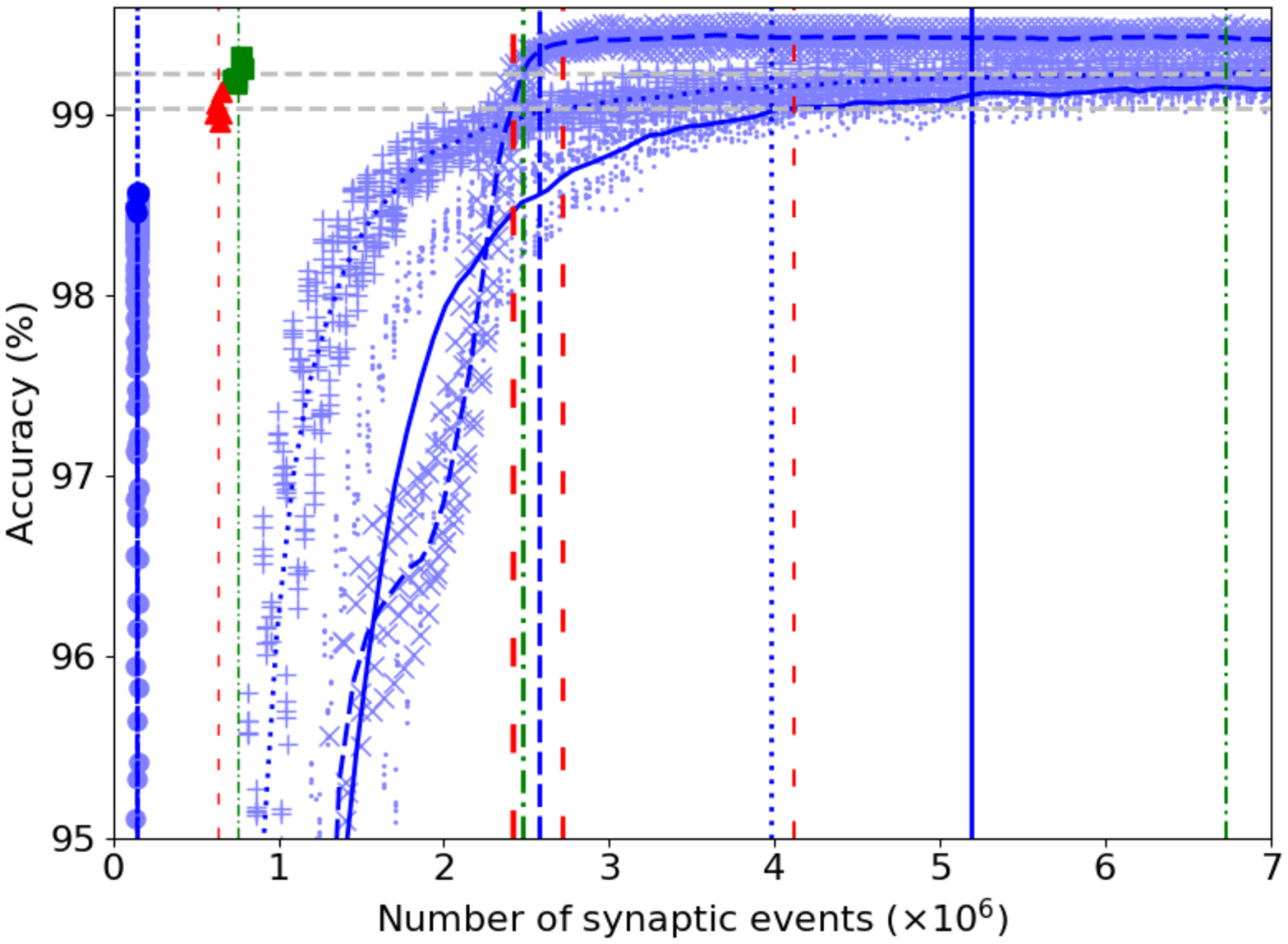}
        \label{fig_computational_cost_accuracy_cnn_small_synaptic_events_zoom_in}}
    \hfil
    \subfloat[]{
        \includegraphics[width=\columnwidth]{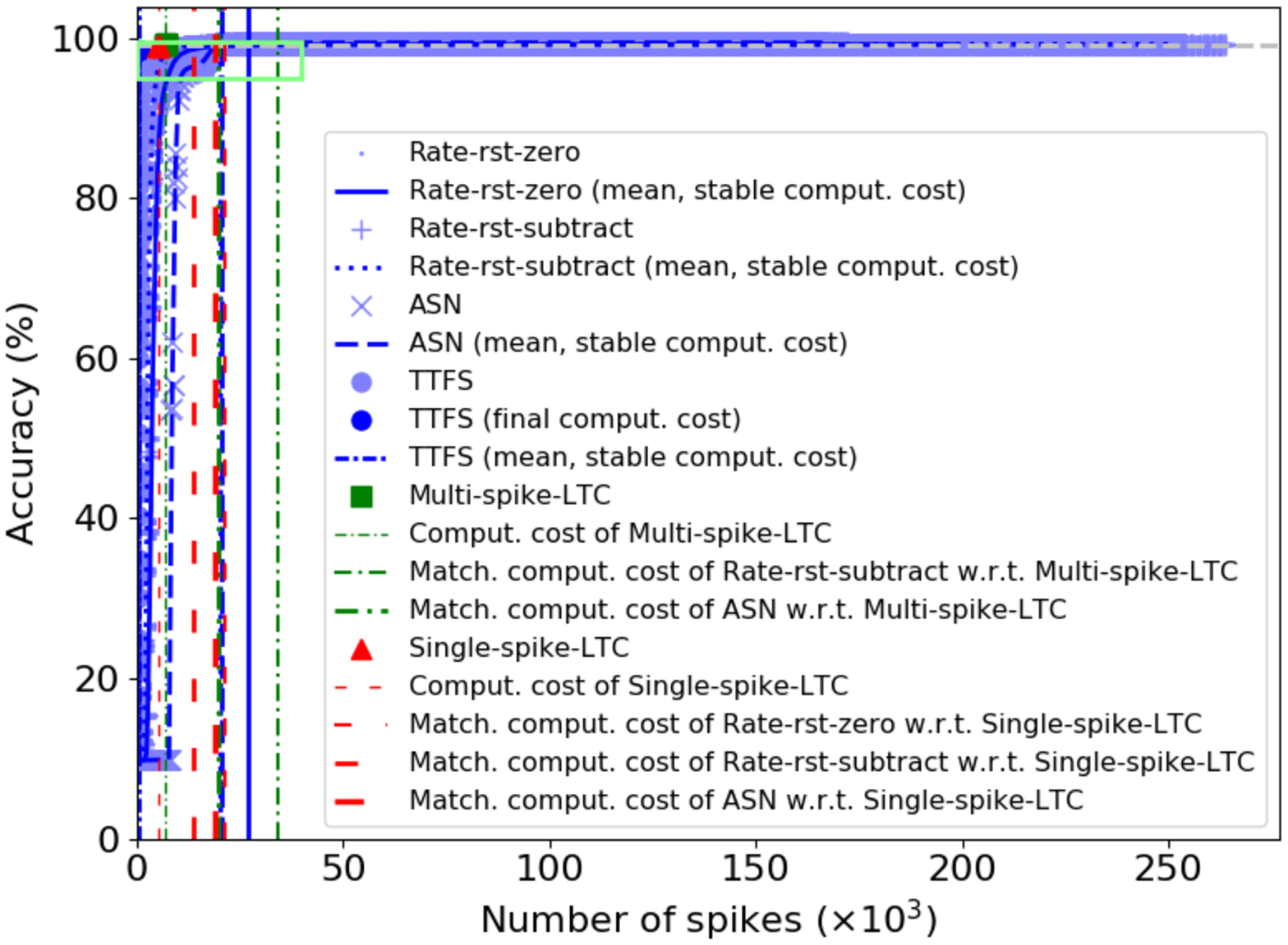}
        \label{fig_computational_cost_accuracy_cnn_small_spikes_complete}}
    \subfloat[]{
        \includegraphics[width=\columnwidth]{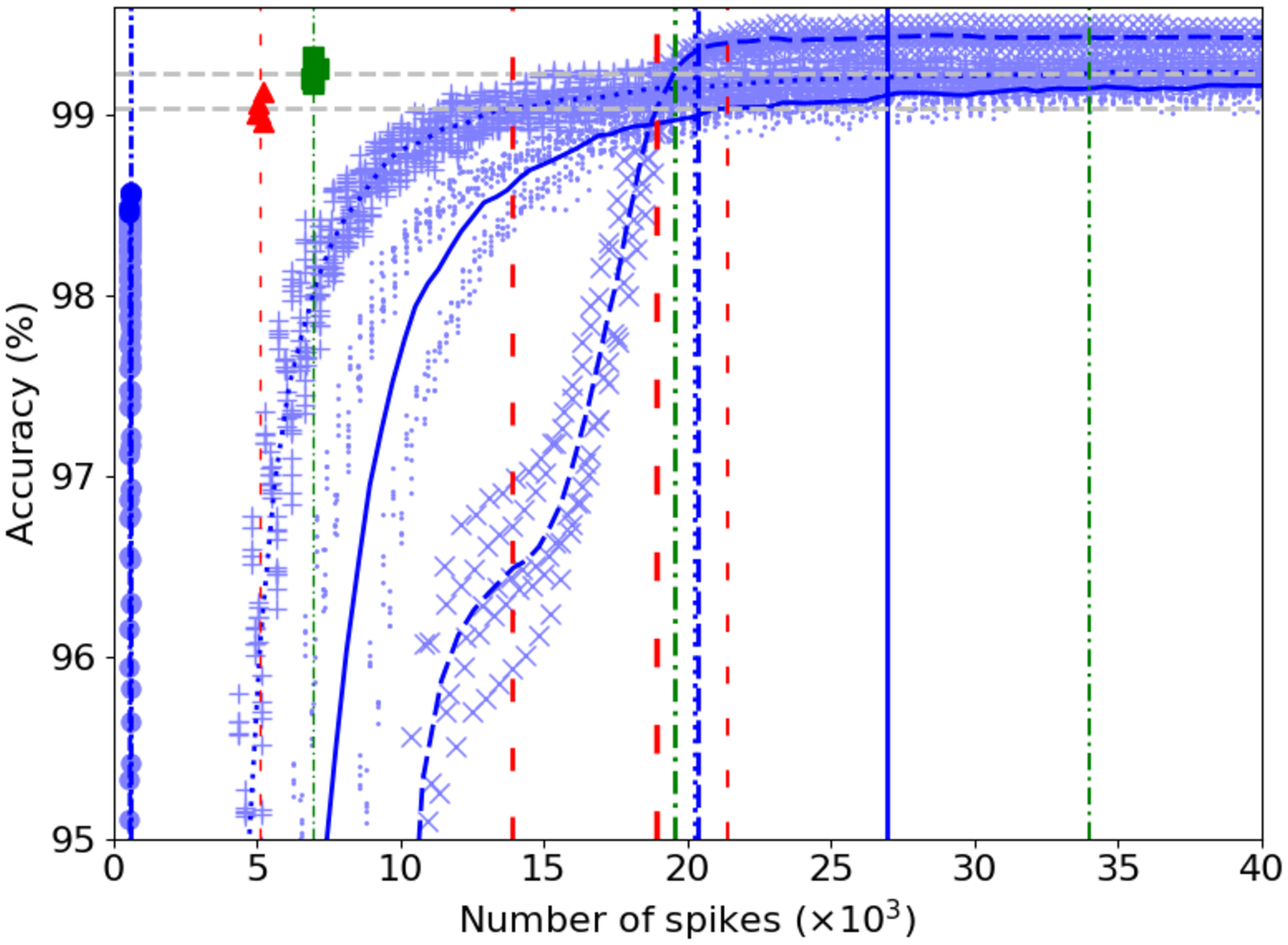}
        \label{fig_computational_cost_accuracy_cnn_small_spikes_zoom_in}}
    \caption[]{Computational costs and accuracies of SNNs with the CNN-small architecture.
        The ``SNN-'' prefix is omitted.
        \subref{fig_computational_cost_accuracy_cnn_small_synaptic_events_zoom_in} is a close-up view of the region in the green box in \subref{fig_computational_cost_accuracy_cnn_small_synaptic_events_complete},
        and \subref{fig_computational_cost_accuracy_cnn_small_spikes_zoom_in} is a close-up view of the region in the green box in \subref{fig_computational_cost_accuracy_cnn_small_spikes_complete}.}
    \label{fig_computational_cost_accuracy_cnn_small}
\end{figure*}

As shown in Figure \ref{fig_computational_cost_accuracy_cnn_small},
SNN-LTC types achieve high test accuracies at low computational costs.
At the same average computational costs, the SNN-rate-IF types and the SNN-ASN type achieve significantly lower average test accuracies ranging from 9.8\% to 98\%.
The average test accuracies of SNN-rate-IF and SNN-ASN increase quickly with increasing computational costs at an early stage of the test runs,
and then fluctuate near their maximum values for a long time until the end of simulation.
The average test accuracy of SNN-TTFS increases rapidly with increasing computational costs at the end of simulation,
when the output layers of the SNNs fire their first output spikes.

In order to compare the ever-changing average computational costs of previous ANN-to-SNN conversion methods with the final average computational costs of the SNN-LTC types,
we find two kinds of \emph{reference computational costs} for each of SNN-rate-IF-rst-zero, SNN-rate-IF-rst-subtract, SNN-ASN, and SNN-TTFS.
One is the \emph{stable computational cost} where the average test accuracy converges to the final average test accuracy.
Specifically, we consider the average test accuracy to have converged if it remains within the $\pm 0.1\%$ range around the final average test accuracy until the end of the simulation time.
The other kind of reference computational costs are the \emph{matching computational costs} w.r.t. to each of the SNN-LTC types,
where the average test accuracy of the SNN-rate-IF, SNN-ASN, or SNN-TTFS type starts to surpass the average test accuracy of the SNN-LTC type.
The reference computational costs are marked with vertical lines in Figure \ref{fig_computational_cost_accuracy_cnn_small}.

Table \ref{tbl_computational_costs_cnn_small} compares computational costs of our ANN-to-SNN conversion methods with those of previous ANN-to-SNN conversion methods, for the CNN-small architecture.
For every SNN-rate-IF type and the SNN-ASN type,
both the stable computational costs and the matching computational costs are shown,
along with the ratios (in percentage) of the SNN-LTC types' computational costs to the reference computational costs.
The matching computational cost of SNN-rate-rst-zero w.r.t. SNN-multi-spike-LTC is not shown,
because the average test accuracy of SNN-multi-spike-LTC is higher than the highest average test accuracy of SNN-rate-rst-zero.
The matching computational costs of SNN-TTFS are not shown for the same reason.

As shown in Table \ref{tbl_computational_costs_cnn_small},
the average computational costs of the SNN-LTC types are much lower than the reference computational costs of the SNN-rate-IF types and the SNN-ASN type.
Compared with the SNN-rate-IF types,
SNN-multi-spike-LTC achieves a similar average test accuracy
while reducing the computational cost by more than 80\% in terms of synaptic events and more than 65\% in terms of spikes;
SNN-single-spike-LTC reduces the computational cost by more than 76\% in terms of synaptic events and more than 63\% in terms of spikes,
at the cost of a decrease of 0.22\% in final average test accuracy.
Compared with the SNN-ASN type,
SNN-multi-spike-LTC reduces the computational cost by more than 69\% in terms of synaptic events and more than 64\% in terms of spikes,
at the cost of a decrease of 0.2\% in final average test accuracy;
SNN-single-spikc-LTC reduces the computational cost by more than 73\% in terms of synaptic events and more than 72\% in terms of spikes,
at the cost of a decrease of 0.4\% in final average test accuracy.
Compared with SNN-single-spike-LTC, SNN-multi-spike-LTC achieves a higher average test accuracy at a higher average computational cost.

\begin{table*}[!t]
\renewcommand{\arraystretch}{1.3}
\caption{Comparison of computational costs of SNN types with the CNN-small architecture.}
\label{tbl_computational_costs_cnn_small}
\centering
\begin{tabular}{|l|cc|cc|}
\hline \hline
& \multicolumn{2}{|c|}{\# Synaptic events} & \multicolumn{2}{|c|}{\# Spikes} \\
\hline \hline
& SNN-multi-spike-LTC & SNN-single-spike-LTC & SNN-multi-spike-LTC & SNN-single-spike-LTC \\
\hline
& $7.56 \times 10^5$ & $6.36 \times 10^5$ & $6.97 \times 10^3$ & $5.13 \times 10^3$ \\
\hline
SNN-rate-rst-zero (stable) & $5.18 \times 10^6$ ($14.58\%$) & $5.18 \times 10^6$ ($12.27\%$) & $2.69 \times 10^4$ ($25.87\%$) & $2.69 \times 10^4$ ($19.03\%$) \\
\hline
SNN-rate-rst-zero (matching) & N/A & $4.11 \times 10^6$ ($15.47\%$) & N/A & $2.14 \times 10^4$ ($23.96\%$) \\
\hline
SNN-rate-rst-subtract (stable) & $3.97 \times 10^6$ ($19.03\%$) & $3.97 \times 10^6$ ($16.02\%$) & $2.02 \times 10^4$ ($34.48\%$) & $2.02 \times 10^4$ ($25.37\%$) \\
\hline
SNN-rate-rst-subtract (matching) & $6.73 \times 10^6$ ($11.23\%$) & $2.71 \times 10^6$ ($23.45\%$) & $3.40 \times 10^4$ ($20.49\%$) & $1.38 \times 10^4$ ($36.93\%$) \\
\hline
SNN-ASN (stable) & $2.58 \times 10^6$ ($29.32\%$) & $2.58 \times 10^6$ ($24.68\%$) & $2.03 \times 10^4$ ($34.28\%$) & $2.03 \times 10^4$ ($25.22\%$) \\
\hline
SNN-ASN (matching) & $2.47 \times 10^6$ ($30.55\%$) & $2.41 \times 10^6$ ($26.40\%$) & $1.95 \times 10^4$ ($35.70\%$) & $1.89 \times 10^4$ ($27.10\%$) \\
\hline
SNN-TTFS (stable) & $\mathbf{1.46 \times 10^5}$ ($515.07\%$) & $\mathbf{1.46 \times 10^5}$ ($433.63\%$) & $\mathbf{5.88 \times 10^2}$ ($1184.28\%$) & $\mathbf{5.88 \times 10^2}$ ($871.45\%$) \\
\hline \hline
\end{tabular}
\end{table*}

Compared with SNN-TTFS, both SNN-LTC types achieve significantly higher average test accuracies,
but at much higher average computational costs in terms of synaptic events.
However, for SNN-TTFS, the number of synaptic events is an underestimate of the true computational cost.
According to the membrane potential update rule (Equation (4) in \cite{Rueckauer18}),
a TTFS neuron keeps track of the synapses which have ever received an input spike,
and adds the sum of the synaptic weights to its membrane potential every time step.
The number of synaptic events accounts for the updates of the sum of synaptic weights,
not the updates of the membrane potential.
As shown in Table \ref{tbl_computational_costs_cnn_small_ttfs},
the number of membrane potential updates (other ADDs) dominates the true computational cost of SNN-TTFS.
The computational costs of the SNN-LTC types are similar to the true computational cost of SNN-TTFS.
The average computational cost of SNN-multi-spike-LTC is 5.20\% higher,
and the average computational cost of SNN-single-spike-LTC is 11.43\% lower.

\begin{table*}[!t]
\renewcommand{\arraystretch}{1.3}
\caption{Comparison of computational costs of SNN-LTC types and the SNN-TTFS type with the CNN-small architecture.}
\label{tbl_computational_costs_cnn_small_ttfs}
\centering
\begin{tabular}{l c c c}
\hline
& \# ADDs for synaptic events & \# Other ADDs & Comput. cost \\
\hline
SNN-TTFS (stable) & $\mathbf{1.46 \times 10^5}$ & $5.72 \times 10^5$ & $7.19 \times 10^5$ \\
SNN-multi-spike-LTC & $7.56 \times 10^5$ & $\mathbf{0}$ & $7.56 \times 10^5$ ($105.20\%$) \\
SNN-single-spike-LTC & $6.36 \times 10^5$ & $\mathbf{0}$ & $\mathbf{6.36 \times 10^5}$ ($\mathbf{88.57\%}$) \\
\hline
\end{tabular}
\end{table*}

Figure \ref{fig_computational_cost_accuracy_cnn_large} shows computational costs and test accuracies of SNNs with the CNN-large architecture.
The SNN-LTC types achieve high test accuracies at low computational costs.
At the average computational costs of the SNN-LTC types, the SNN-rate-IF types and the SNN-ASN type achieve very poor average test accuracies around 9.8\%.

\begin{figure*}[!t]
    \centering
    \subfloat[]{
        \includegraphics[width=\columnwidth]{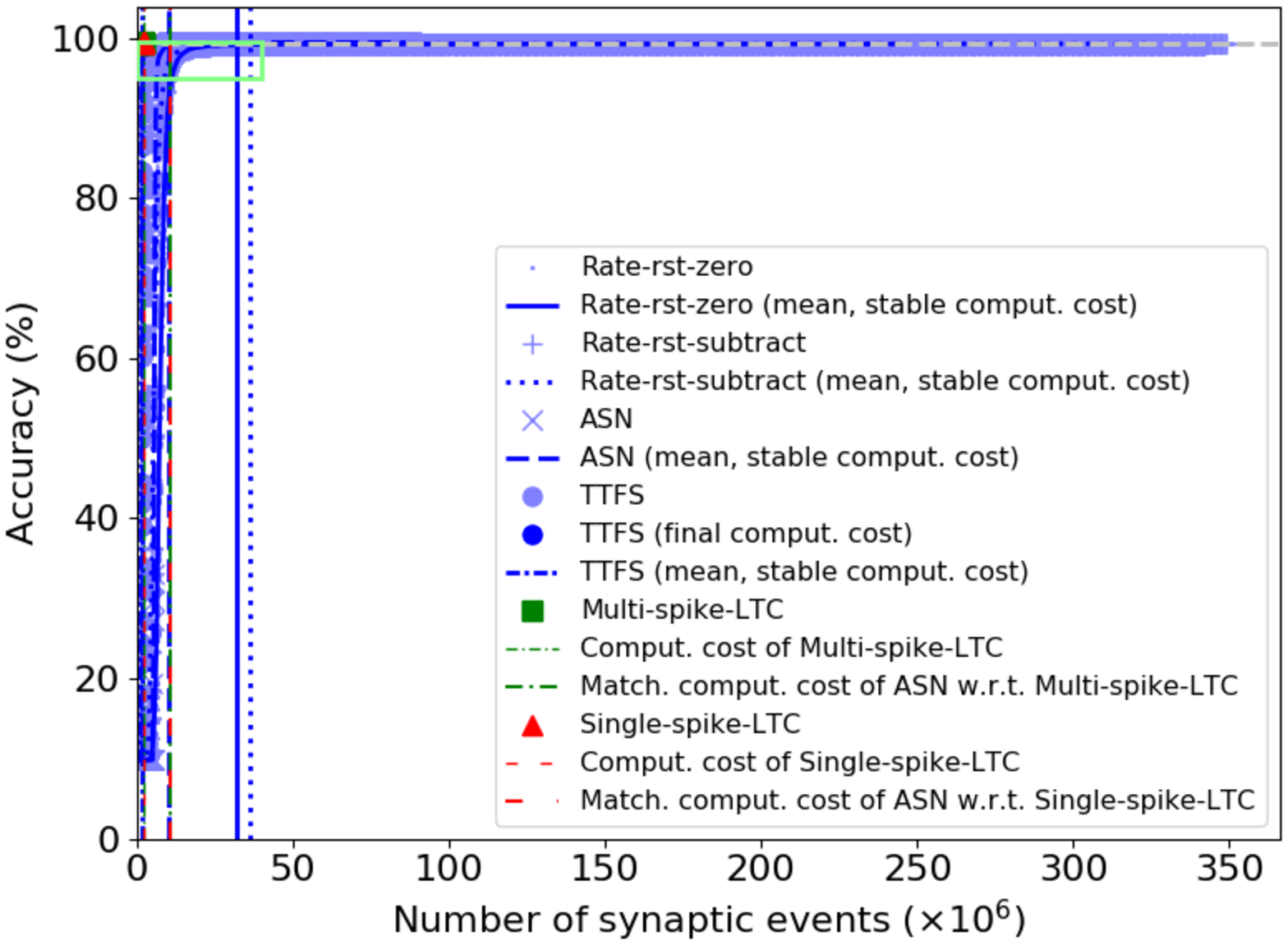}
        \label{fig_computational_cost_accuracy_cnn_large_synaptic_events_complete}}
    \subfloat[]{
        \includegraphics[width=\columnwidth]{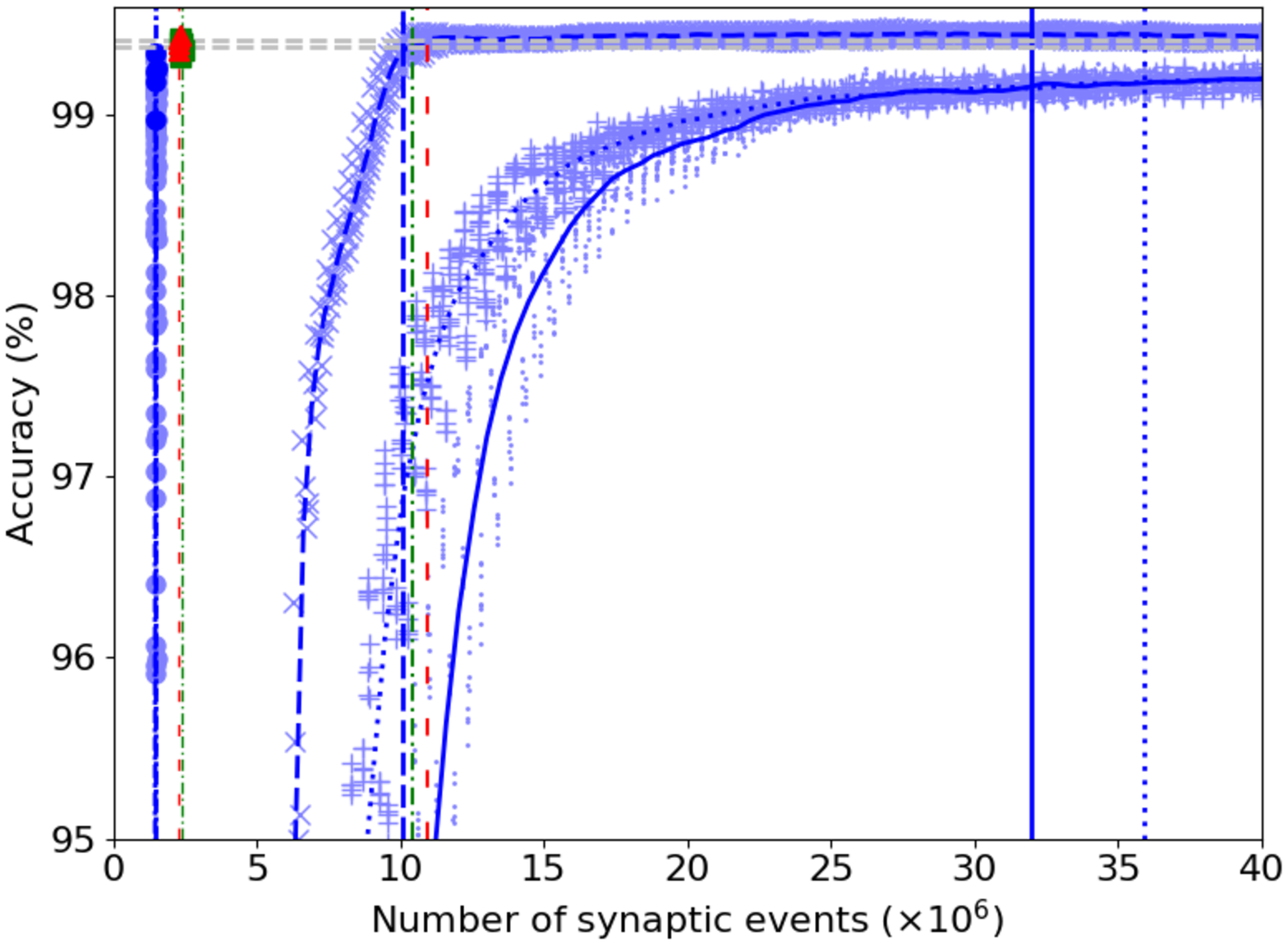}
        \label{fig_computational_cost_accuracy_cnn_large_synaptic_events_zoom_in}}
    \hfil
    \subfloat[]{
        \includegraphics[width=\columnwidth]{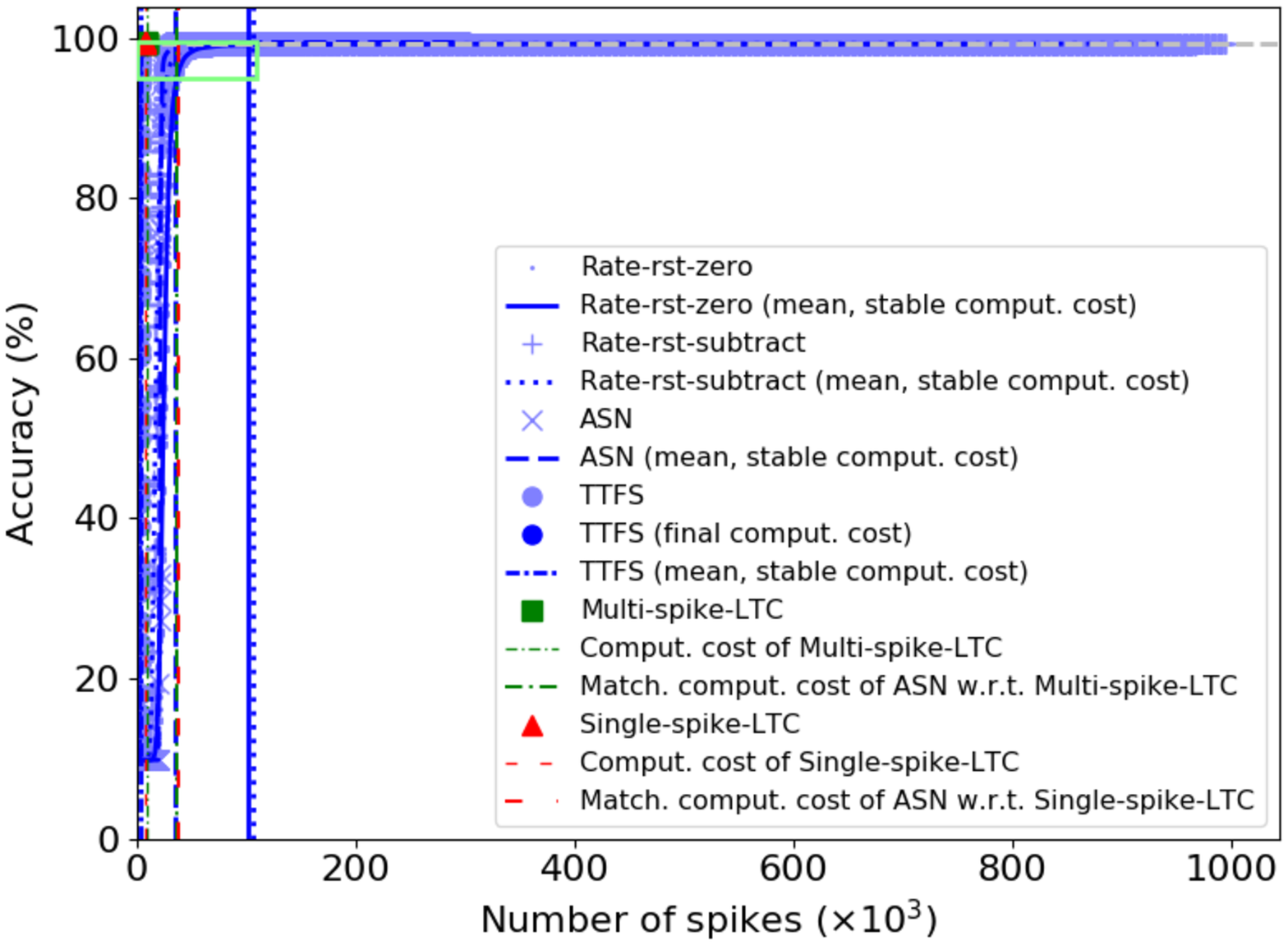}
        \label{fig_computational_cost_accuracy_cnn_large_spikes_complete}}
    \subfloat[]{
        \includegraphics[width=\columnwidth]{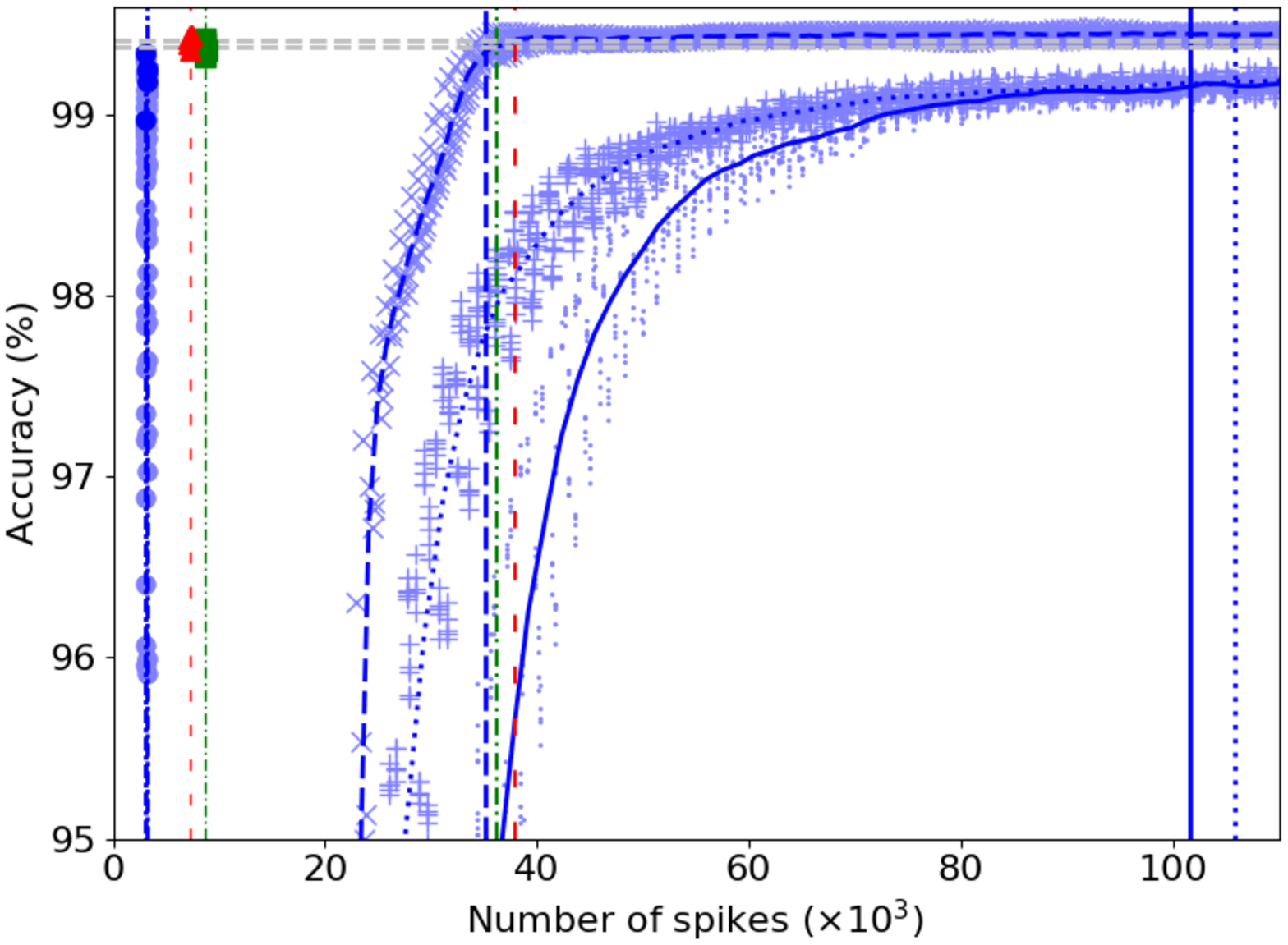}
        \label{fig_computational_cost_accuracy_cnn_large_spikes_zoom_in}}
    \caption[]{Computational costs and accuracies of SNNs with the CNN-large architecture.
        The ``SNN-'' prefix is omitted.
        \subref{fig_computational_cost_accuracy_cnn_large_synaptic_events_zoom_in} is a close-up view of the region in the green box in \subref{fig_computational_cost_accuracy_cnn_large_synaptic_events_complete},
        and \subref{fig_computational_cost_accuracy_cnn_large_spikes_zoom_in} is a close-up view of the region in the green box in \subref{fig_computational_cost_accuracy_cnn_large_spikes_complete}.}
    \label{fig_computational_cost_accuracy_cnn_large}
\end{figure*}

Similar to Table \ref{tbl_computational_costs_cnn_small},
Table \ref{tbl_computational_costs_cnn_large} compares computational costs of our ANN-to-SNN conversion methods with those of previous ANN-to-SNN conversion methods, for the CNN-large architecture.
For the SNN-rate-IF types and the SNN-TTFS type, only the stable computational costs are shown,
since the average test accuracies of the SNN-LTC types are higher than the highest average test accuracies of these types.
Compared with the SNN-rate-IF types,
the SNN-LTC types achieve higher average test accuracies
while reducing the computational cost by more than 92\% in terms of synaptic events and more than 91\% in terms of spikes.
Compared with the SNN-ASN type,
the SNN-LTC types reduce the computational cost by more than 76\% in terms of synaptic events and more than 75\% in terms of spikes,
at the cost of a slight decrease of less than 0.1\% in final average test accuracy.
SNN-single-spike-LTC slightly outperforms SNN-multi-spike-LTC at a lower computational cost.

\begin{table*}[!t]
\renewcommand{\arraystretch}{1.3}
\caption{Comparison of computational costs of SNN types with the CNN-large architecture.}
\label{tbl_computational_costs_cnn_large}
\centering
\begin{tabular}{|l|cc|cc|}
\hline \hline
& \multicolumn{2}{|c|}{\# Synaptic events} & \multicolumn{2}{|c|}{\# Spikes} \\
\hline \hline
& SNN-multi-spike-LTC & SNN-single-spike-LTC & SNN-multi-spike-LTC & SNN-single-spike-LTC \\
\hline
& $2.37 \times 10^6$ & $2.29 \times 10^6$ & $8.65 \times 10^3$ & $7.22 \times 10^3$ \\
\hline
SNN-rate-rst-zero (stable) & $3.19 \times 10^7$ ($7.43\%$) & $3.19 \times 10^7$ ($7.17\%$) & $1.01 \times 10^5$ ($8.51\%$) & $1.01 \times 10^5$ ($7.11\%$) \\
\hline
SNN-rate-rst-subtract (stable) & $3.59 \times 10^7$ ($6.60\%$) & $3.59 \times 10^7$ ($6.38\%$) & $1.05 \times 10^5$ ($8.17\%$) & $1.05 \times 10^5$ ($6.82\%$) \\
\hline
SNN-ASN (stable) & $1.00 \times 10^7$ ($23.62\%$) & $1.00 \times 10^7$ ($22.82\%$) & $3.50 \times 10^4$ ($24.67\%$) & $3.50 \times 10^4$ ($20.60\%$) \\
\hline
SNN-ASN (matching) & $1.04 \times 10^7$ ($22.73\%$) & $1.09 \times 10^7$ ($21.00\%$) & $3.62 \times 10^4$ ($23.87\%$) & $3.79 \times 10^4$ ($19.04\%$) \\
\hline
SNN-TTFS (stable) & $\mathbf{1.46 \times 10^6}$ ($161.98\%$) & $\mathbf{1.46 \times 10^6}$ ($156.47\%$) & $\mathbf{3.10 \times 10^3}$ ($278.53\%$) & $\mathbf{3.10 \times 10^3}$ ($232.60\%$) \\
\hline \hline
\end{tabular}
\end{table*}

Both SNN-LTC types achieve higher average test accuracies than the SNN-TTFS type.
As shown in Table \ref{tbl_computational_costs_cnn_large_ttfs},
SNN-multi-spike-LTC and SNN-single-spike-LTC reduce the average computational cost by 41.22\% and 43.22\%, respectively.

\begin{table*}[!t]
\renewcommand{\arraystretch}{1.3}
\caption{Comparison of computational costs of SNN-LTC types and the SNN-TTFS type with the CNN-large architecture.}
\label{tbl_computational_costs_cnn_large_ttfs}
\centering
\begin{tabular}{l c c c}
\hline
& \# ADDs for synaptic events & \# Other ADDs & Comput. cost \\
\hline
SNN-TTFS (stable) & $\mathbf{1.46 \times 10^6}$ & $2.57 \times 10^6$ & $4.03 \times 10^6$ \\
SNN-multi-spike-LTC & $2.37 \times 10^6$ & $\mathbf{0}$ & $2.37 \times 10^6$ ($58.78\%$) \\
SNN-single-spike-LTC & $2.29 \times 10^6$ & $\mathbf{0}$ & $\mathbf{2.29 \times 10^6}$ ($\mathbf{56.78\%}$) \\
\hline
\end{tabular}
\end{table*}

\section{Conclusions}

In this work, we propose an ANN-to-SNN conversion method based on novel Logarithmic Temporal Coding (LTC), and the Exponentiate-and-Fire (EF) neuron model.
Moreover, we introduce the approximation errors of LTC into the ANN,
and train the ANN to compensate for the approximation errors,
eliminating most of performance drop due to ANN-to-SNN conversion.
The experimental results show that the proposed method achieves competitive performance at a significantly lower computational cost.

In future work, we are going to explore the combination of our logarithmic temporal coding, which sparsifies spike trains in time,
and regularization techniques that sparsify spike trains across spiking neurons.
Sparsifying spike trains across both space and time may help achieve further computational efficiency.

\appendices
\section{An exponentiate-and-fire neuron generates a logarithmic temporal coding spike train}
\label{sec_consistency_of_ef_neuron_spike_generation_with_ltc}

In this section, we prove Lemma \ref{lem_ef_neuron_spike_generation_consistency} in Section \ref{sec_spike_generation}.

We observe that an EF neuron doubles its membrane potential every time step
if the neuron does not receive or fire any spike, as Lemma \ref{lem_exponentially_growing_membrane_potential} states.

\begin{lemma}
Let $t_0, t \in \mathbb{Z}$ be two time steps, where $t_0 < t$.
The pre-reset membrane potential of an EF neuron $V_m^-(t) = V_m(t_0) \cdot 2^{t - t_0}$ if the following conditions hold:
\begin{enumerate}
\item the EF neuron does not receive any input spike during the time interval $\{ t_0 + 1, \ldots , t \}$, and
\item the EF neuron does not fire any output spike during the time interval $\{ t_0 + 1, \ldots , t - 1 \}$.
\end{enumerate}
\label{lem_exponentially_growing_membrane_potential}
\end{lemma}

\begin{proof}
Lemma \ref{lem_exponentially_growing_membrane_potential} follows from
the definition of the membrane potential (Eqn. \ref{eq_membrane_potential}),
the definition of the pre-reset membrane potential (Eqn. \ref{eq_pre_reset_membrane_potential}),
and the exponentially growing postsynaptic potential and afterhyperpolarizing potential kernels (Eqn. \ref{eq_psp_kernel}, \ref{eq_reset_by_subtraction} and \ref{eq_reset_to_zero}).
\end{proof}

With Lemma \ref{lem_exponentially_growing_membrane_potential},
we prove Lemma \ref{lem_ef_neuron_spike_generation_consistency} below.

\begin{proof}
Depending on the value of $V_m^-(T^{in} - 1)$,
there are four cases:
$V_m^-(T^{in} - 1) \leq 0$,
$0 < V_m^-(T^{in} - 1) < 2^{e^{out}_{min}}$,
$2^{e^{out}_{min}} \leq V_m^-(T^{in} - 1) < 2^{e^{out}_{max} + 1}$,
and $V_m^-(T^{in} - 1) \geq 2^{e^{out}_{max} + 1}$.

If $V_m^-(T^{in} - 1) \leq 0$, the logarithmic approximation of $\max (V_m^-(T^{in} - 1), 0)$ is 0,
and the desired LTC spike train contains no spikes.
For the EF neuron, by Lemma \ref{lem_exponentially_growing_membrane_potential},
$V_m^-(t)$ remains zero or negative during the output time window,
and the neuron does not fire any output spike during this time interval.
Hence, the neuron generates the desired LTC spike train within its output time window.

If $0 < V_m^-(T^{in} - 1) < 2^{e^{out}_{min}}$,
with exponent range $\{ e^{out}_{min}, \ldots , e^{out}_{max} \}$,
the logarithmic approximation of $\max (V_m^-(T^{in} - 1), 0)$ is 0,
and the desired LTC spike train contains no spikes.

For the EF neuron, by Lemma \ref{lem_exponentially_growing_membrane_potential},
the first output spike time $t^{out}_0$ satisfies the following condition:
\begin{equation}
    V_m^-(t^{out}_0) = V_m^-(T^{in} - 1) \cdot 2^{t^{out}_0 - (T^{in} - 1)} \geq V_{th}
    \label{ie_first_output_spike_time}
\end{equation}
Solving the equation, we have
\begin{equation}
    t^{out}_0 = e^{out}_{max} - \lfloor \log_2{V_m^-(T^{in} - 1)} \rfloor + (T^{in} - 1)
    \label{eq_first_output_spike_time_moderate_activation}
\end{equation}
since $V_m^-(T^{in} - 1) < 2^{e^{out}_{min}}$, $t^{out}_0 > T^{in} + T^{out} - 2$.
In other words, the neuron fires the first output spike after the end of its output time window.
Hence, the neuron generates the desired LTC spike train within its output time window.

For the remaining cases,
we derive the spike times of the desired LTC spike train and the output spike times of the EF neuron,
and show that the output spike train of the EF neuron is consistent with the desired LTC spike train at the end of this proof.

The case where $2^{e^{out}_{min}} \leq V_m^-(T^{in} - 1) < 2^{e^{out}_{max} + 1}$ corresponds to
the case of Eqn. \ref{eq_logarithmic_approximation_multi_power} where $2^{e_{min}} \leq a < 2^{e_{max} + 1}$.
In this case, Eqn. \ref{eq_logarithmic_approximation_multi_power} can be formulated as
\begin{equation}
    \tilde{a} = \sum_{k}{2^{e_k}}
    \label{eq_logarithmic_approximation_multi_power_recursive}
\end{equation}
\begin{equation}
    e_k =
    \begin{cases}
        \lfloor \log_2{a} \rfloor & \text{if } k = 0, \\
        \lfloor \log_2{(a - \sum_{k' = 0}^{k - 1}{2^{e_{k'}}})} \rfloor & \text{if } k > 0
    \end{cases}
    \label{eq_logarithmic_approximation_multi_power_recursive_exponents}
\end{equation}
\begin{equation}
    \forall k, e_k \geq e_{min}
    \label{ie_logarithmic_approximation_multi_power_recursive_exponents_condition}
\end{equation}
where the sum in Eqn. \ref{eq_logarithmic_approximation_multi_power_recursive} runs across exponents $\{ e_0, e_1, \ldots \}$ from $\lfloor \log_2{a} \rfloor$ to the smallest $e_k \geq e_{min}$.
Note that $2^{e_0}$ gives the single-power LA of $a$.
By substituting $e^{out}_k = e^{out}_{max} - t^{out}_k$ and $a = V_m^-(T^{in} - 1)$ into
Eqn. \ref{eq_logarithmic_approximation_multi_power_recursive_exponents},
we derive the spike times of the desired LTC spike train:
\begin{equation}
    t^{out}_k = e^{out}_{max} - \lfloor \log_2{(V_m^-(T^{in} - 1) - \sum_{k' = 0}^{k - 1}{2^{e^{out}_{max} - t^{out}_{k'}}})} \rfloor
    \label{eq_ltc_spike_times_moderate_activation}
\end{equation}
where $t^{out}_k$ is the $(k + 1)$-th output spike time.
For both multi-spike LTC and single-spike LTC, Eqn. \ref{eq_ltc_spike_times_moderate_activation} gives the first spike time $t^{out}_0$.
For multi-spike LTC, Eqn. \ref{eq_ltc_spike_times_moderate_activation} also gives subsequent spike times.
By further substituting
Eqn. \ref{eq_logarithmic_approximation_multi_power_recursive_exponents} and $e^{out}_{min} = e^{out}_{max} - (T^{out} - 1)$
into Inequality \ref{ie_logarithmic_approximation_multi_power_recursive_exponents_condition},
we derive constraints on the spike times:
\begin{equation}
    \forall k,
    V_m^-(T^{in} - 1) - \sum_{k' = 0}^{k - 1}{2^{e^{out}_{max} - t^{out}_{k'}}} \geq 2^{e^{out}_{max} - (T^{out} - 1)}
    \label{ie_ltc_spike_times_moderate_activation_condition_membrane_potential}
\end{equation}
\begin{equation}
    \forall k,
    t^{out}_k \leq T^{out} - 1
    \label{ie_ltc_spike_times_condition_spike_time}
\end{equation}

For the EF neuron, every output spike time $t^{out}_k$ within the output time window satisfies the following conditions:
\begin{equation}
    V_m^-(t^{out}_k) \geq V_{th}
    \label{ie_output_spike_time_condition_membrane_potential}
\end{equation}
\begin{equation}
    t^{out}_k \leq T^{in} + T^{out} - 2
    \label{ie_output_spike_time_condition_spike_time}
\end{equation}

The first output spike may be fired either at time $t^{out}_0 = T^{in} - 1$, if $V_m^-(T^{in} - 1) \geq 2^{e^{out}_{max}}$; or at time $t^{out}_0 > T^{in} - 1$ (Lemma \ref{lem_exponentially_growing_membrane_potential}), if $V_m^-(T^{in} - 1) < 2^{e^{out}_{max}}$.
In both cases, the first output spike time $t^{out}_0$ satisfies Eqns. \ref{ie_first_output_spike_time} and \ref{eq_first_output_spike_time_moderate_activation}.

In the case of single-spike LTC, the EF neuron fires a single output spike at $t = t^{out}_0$.
In the case of multi-spike LTC, the EF neuron may fire subsequent output spikes.
Consider every two consecutive output spike times $t^{out}_{k - 1}$ and $t^{out}_k$, where $t^{out}_{k - 1} < t^{out}_k$.
By Lemma \ref{lem_exponentially_growing_membrane_potential},
the pre-reset membrane potentials $V_m^-(t^{out}_k)$ can be formulated as
\begin{align}
    V_m^-(t^{out}_k) & = (V_m^-(t^{out}_{k-1}) - 2^{e^{out}_{max}}) \cdot 2^{t^{out}_k - t^{out}_{k-1}} \\
    V_m^-(t^{out}_0) & = V_m^-(T^{in} - 1) \cdot 2^{t^{out}_0 - (T^{in} - 1)}
\end{align}
Solving the recurrence relation above, we have
\begin{equation}
    V_m^-(t^{out}_k) = 2^{t^{out}_k} \cdot (V_m^-(T^{in} - 1) \cdot 2^{-(T^{in} - 1)} - \sum_{k' = 0}^{k - 1}{2^{e^{out}_{max} - t^{out}_{k'}}})
    \label{eq_pre_reset_membrane_potential_at_spike_time}
\end{equation}
By substituting Eqn. \ref{eq_pre_reset_membrane_potential_at_spike_time} into
Inequality \ref{ie_output_spike_time_condition_membrane_potential} and
considering Inequality \ref{ie_output_spike_time_condition_spike_time},
we have
\begin{equation}
    V_m^-(T^{in} - 1) -  \sum_{k' = 0}^{k - 1}{2^{e^{out}_{max} - (t^{out}_{k'} - (T^{in} - 1))}} \geq 2^{e^{out}_{max} - (T^{out} - 1)}
    \label{ie_output_spike_time_moderate_activation_condition_membrane_potential}
\end{equation}
By substituting Eqn. \ref{eq_pre_reset_membrane_potential_at_spike_time} into
Inequality \ref{ie_output_spike_time_condition_membrane_potential} and
solving the resulting inequality for the minimum integer value for $t^{out}_k$,
we have
\begin{align}
    \begin{split}
        & t^{out}_k - (T^{in} - 1) = \\
        & e^{out}_{max} - \lfloor \log_2{(V_m^-(T^{in} - 1) - \sum_{k' = 0}^{k - 1}{2^{e^{out}_{max} - (t^{out}_{k'} - (T^{in} - 1))}})} \rfloor
    \end{split}
    \label{eq_output_spike_time_recurrence}
\end{align}

The case where $V_m^-(T^{in} - 1) \geq 2^{e^{out}_{max} + 1}$ corresponds to
the case of Eqn. \ref{eq_logarithmic_approximation_multi_power} where $a \geq 2^{e_{max} + 1}$.
In this case, Eqn. \ref{eq_logarithmic_approximation_multi_power} can be formulated as
\begin{equation}
    \tilde{a} = \sum_{k = 0}^{e_{max} - e_{min}}{2^{e_k}}
    \label{eq_logarithmic_approximation_multi_power_excessive_activation}
\end{equation}
\begin{equation}
    e_k = e_{max} - k
    \label{eq_logarithmic_approximation_multi_power_excessive_activation_exponents}
\end{equation}
Note that $2^{e_0}$ gives the single-power LA of $a$.
By substituting $e^{out}_k = e^{out}_{max} - t^{out}_k$ into
Eqn. \ref{eq_logarithmic_approximation_multi_power_excessive_activation_exponents},
we derive the spike times of the desired LTC spike train:
\begin{equation}
    \forall k \in \{ 0, \ldots , T^{out} - 1 \} ,
    t^{out}_k = k
    \label{eq_ltc_spike_times_excessive_activation}
\end{equation}
For both multi-spike LTC and single-spike LTC, Eqn. \ref{eq_ltc_spike_times_excessive_activation} gives the first spike time $t^{out}_0 = 0$.
For multi-spike LTC, Eqn. \ref{eq_ltc_spike_times_excessive_activation} also gives subsequent spike times.

For the EF neuron,
since $V_m^-(T^{in} - 1) \geq 2^{e^{out}_{max} + 1} > V_{th}$,
the first output spike time is
\begin{equation}
    t^{out}_0 = T^{in} - 1
\end{equation}

In the case of single-spike LTC, the EF neuron fires only a single output spike.
In the case of multi-spike LTC, suppose the EF neuron fires an output spike at the time step $t^{out}_k$.
By Lemma \ref{lem_exponentially_growing_membrane_potential},
\begin{equation}
    V_m^-(t^{out}_k + 1) = 2(V_m^-(t^{out}_k) - 2^{e^{out}_{max}})
\end{equation}
It is easy to see that, if $V_m^-(t^{out}_k) \geq 2^{e^{out}_{max} + 1}$,
then $V_m^-(t^{out}_k + 1) \geq 2^{e^{out}_{max} + 1} > V_{th}$,
and $t^{out}_{k + 1} = t^{out}_k + 1$ will be the next output spike time.
Since $V_m^-(t^{out}_0) = V_m^-(T^{in} - 1) \geq 2^{e^{out}_{max} + 1}$,
the EF neuron fires an output spike at every time step within its output time window.
Hence,
\begin{equation}
    \forall k \in \{ 0, \ldots , T^{out} - 1 \} ,
    t^{out}_k = T^{in} - 1 + k
    \label{eq_output_spike_time_excessive_activation}
\end{equation}

By comparing Eqn. \ref{eq_first_output_spike_time_moderate_activation} and \ref{eq_output_spike_time_recurrence} with Eqn. \ref{eq_ltc_spike_times_moderate_activation},
Inequality \ref{ie_output_spike_time_moderate_activation_condition_membrane_potential} with Inequality \ref{ie_ltc_spike_times_moderate_activation_condition_membrane_potential},
Inequality \ref{ie_output_spike_time_condition_spike_time} with Inequality \ref{ie_ltc_spike_times_condition_spike_time},
and Eqn. \ref{eq_output_spike_time_excessive_activation} with Eqn. \ref{eq_ltc_spike_times_excessive_activation},
it can be seen that the output spike train of the EF neuron within its output time window is consistent with the desired LTC spike train,
except that every output spike time of the EF neuron is $T^{in} - 1$ larger than the corresponding spike time of the desired LTC spike train.
The difference is due to the fact that the output time window of the EF neuron starts at the time step $T^{in} - 1$.

Therefore, in all cases, the EF neuron generates an LTC spike train that encodes $\max (V_m^-(T^{in} - 1), 0)$ within its output time window, completing the proof.
\end{proof}




\ifCLASSOPTIONcaptionsoff
  \newpage
\fi



\bibliographystyle{IEEEtran}
\bibliography{mybibfile}

\begin{thebibliography}{10}
\providecommand{\url}[1]{#1}
\csname url@samestyle\endcsname
\providecommand{\newblock}{\relax}
\providecommand{\bibinfo}[2]{#2}
\providecommand{\BIBentrySTDinterwordspacing}{\spaceskip=0pt\relax}
\providecommand{\BIBentryALTinterwordstretchfactor}{4}
\providecommand{\BIBentryALTinterwordspacing}{\spaceskip=\fontdimen2\font plus
\BIBentryALTinterwordstretchfactor\fontdimen3\font minus
  \fontdimen4\font\relax}
\providecommand{\BIBforeignlanguage}[2]{{%
\expandafter\ifx\csname l@#1\endcsname\relax
\typeout{** WARNING: IEEEtran.bst: No hyphenation pattern has been}%
\typeout{** loaded for the language `#1'. Using the pattern for}%
\typeout{** the default language instead.}%
\else
\language=\csname l@#1\endcsname
\fi
#2}}
\providecommand{\BIBdecl}{\relax}
\BIBdecl

\bibitem{Hu2017}
\BIBentryALTinterwordspacing
J.~Hu, L.~Shen, and G.~Sun, ``Squeeze-and-excitation networks,'' \emph{CoRR},
  vol. abs/1709.01507, 2017. [Online]. Available:
  \url{http://arxiv.org/abs/1709.01507}
\BIBentrySTDinterwordspacing

\bibitem{Kowsari2017}
\BIBentryALTinterwordspacing
K.~Kowsari, D.~E. Brown, M.~Heidarysafa, K.~J. Meimandi, M.~S. Gerber, and
  L.~E. Barnes, ``Hdltex: Hierarchical deep learning for text classification,''
  \emph{CoRR}, vol. abs/1709.08267, 2017. [Online]. Available:
  \url{http://arxiv.org/abs/1709.08267}
\BIBentrySTDinterwordspacing

\bibitem{Cazenave2018}
T.~Cazenave, ``Residual networks for computer go,'' \emph{IEEE Transactions on
  Games}, vol.~10, no.~1, pp. 107--110, March 2018.

\bibitem{gutig2006tempotron}
R.~G\"{u}tig and H.~Sompolinsky, ``{The tempotron: a neuron that learns spike
  timing-based decisions},'' \emph{Nature neuroscience}, vol.~9, no.~3, pp.
  420--428, 2006.

\bibitem{Ponulak2010}
F.~Ponulak and A.~Kasiński, ``Supervised learning in spiking neural networks
  with resume: Sequence learning, classification, and spike shifting,''
  \emph{Neural Computation}, vol.~22, no.~2, pp. 467--510, Feb 2010.

\bibitem{Florian2012}
\BIBentryALTinterwordspacing
R.~V. Florian, ``The chronotron: A neuron that learns to fire temporally
  precise spike patterns,'' \emph{PLOS ONE}, vol.~7, no.~8, pp. 1--27, 08 2012.
  [Online]. Available: \url{https://doi.org/10.1371/journal.pone.0040233}
\BIBentrySTDinterwordspacing

\bibitem{mohemmed2012span}
A.~Mohemmed, S.~Schliebs, S.~Matsuda, and N.~Kasabov, ``{Span: Spike pattern
  association neuron for learning spatio-temporal spike patterns},''
  \emph{International Journal of Neural Systems}, vol.~22, no.~04, p. 1250012,
  2012.

\bibitem{Neftci13}
\BIBentryALTinterwordspacing
E.~Neftci, S.~Das, B.~Pedroni, K.~Kreutz-Delgado, and G.~Cauwenberghs,
  ``Event-driven contrastive divergence for spiking neuromorphic systems,''
  \emph{Frontiers in Neuroscience}, vol.~7, p. 272, 2014. [Online]. Available:
  \url{https://www.frontiersin.org/article/10.3389/fnins.2013.00272}
\BIBentrySTDinterwordspacing

\bibitem{Diehl15STDP}
\BIBentryALTinterwordspacing
P.~Diehl and M.~Cook, ``Unsupervised learning of digit recognition using
  spike-timing-dependent lasticity,'' \emph{Frontiers in Computational
  Neuroscience}, vol.~9, p.~99, 2015. [Online]. Available:
  \url{https://www.frontiersin.org/article/10.3389/fncom.2015.00099}
\BIBentrySTDinterwordspacing

\bibitem{Bohte2002}
\BIBentryALTinterwordspacing
S.~M. Bohte, J.~N. Kok, and H.~L. Poutré, ``Error-backpropagation in
  temporally encoded networks of spiking neurons,'' \emph{Neurocomputing},
  vol.~48, no. 1–4, pp. 17 -- 37, 2002. [Online]. Available:
  \url{http://www.sciencedirect.com/science/article/pii/S0925231201006580}
\BIBentrySTDinterwordspacing

\bibitem{McKennoch2006}
S.~McKennoch, D.~Liu, and L.~G. Bushnell, ``Fast modifications of the spikeprop
  algorithm,'' in \emph{The 2006 IEEE International Joint Conference on Neural
  Network Proceedings}, July 2006, pp. 3970--3977.

\bibitem{Booij2005}
\BIBentryALTinterwordspacing
O.~Booij and H.~tat Nguyen, ``A gradient descent rule for spiking neurons
  emitting multiple spikes,'' \emph{Information Processing Letters}, vol.~95,
  no.~6, pp. 552 -- 558, 2005, applications of Spiking Neural Networks.
  [Online]. Available:
  \url{http://www.sciencedirect.com/science/article/pii/S0020019005001560}
\BIBentrySTDinterwordspacing

\bibitem{GhoshDastidar2009}
\BIBentryALTinterwordspacing
S.~Ghosh-Dastidar and H.~Adeli, ``A new supervised learning algorithm for
  multiple spiking neural networks with application in epilepsy and seizure
  detection,'' \emph{Neural Networks}, vol.~22, no.~10, pp. 1419 -- 1431, 2009.
  [Online]. Available:
  \url{http://www.sciencedirect.com/science/article/pii/S0893608009000653}
\BIBentrySTDinterwordspacing

\bibitem{Xu2013}
\BIBentryALTinterwordspacing
Y.~Xu, X.~Zeng, L.~Han, and J.~Yang, ``A supervised multi-spike learning
  algorithm based on gradient descent for spiking neural networks,''
  \emph{Neural Networks}, vol.~43, pp. 99 -- 113, 2013. [Online]. Available:
  \url{http://www.sciencedirect.com/science/article/pii/S0893608013000440}
\BIBentrySTDinterwordspacing

\bibitem{OConnor16}
\BIBentryALTinterwordspacing
P.~O'Connor and M.~Welling, ``Deep spiking networks,'' \emph{CoRR}, vol.
  abs/1602.08323, 2016. [Online]. Available:
  \url{http://arxiv.org/abs/1602.08323}
\BIBentrySTDinterwordspacing

\bibitem{Lee16}
\BIBentryALTinterwordspacing
J.~H. Lee, T.~Delbruck, and M.~Pfeiffer, ``Training deep spiking neural
  networks using backpropagation,'' \emph{Frontiers in Neuroscience}, vol.~10,
  p. 508, 2016. [Online]. Available:
  \url{https://www.frontiersin.org/article/10.3389/fnins.2016.00508}
\BIBentrySTDinterwordspacing

\bibitem{Cao15}
\BIBentryALTinterwordspacing
Y.~Cao, Y.~Chen, and D.~Khosla, ``Spiking deep convolutional neural networks
  for energy-efficient object recognition,'' \emph{International Journal of
  Computer Vision}, vol. 113, no.~1, pp. 54--66, May 2015. [Online]. Available:
  \url{https://doi.org/10.1007/s11263-014-0788-3}
\BIBentrySTDinterwordspacing

\bibitem{Diehl15Conversion}
P.~U. Diehl, D.~Neil, J.~Binas, M.~Cook, S.~C. Liu, and M.~Pfeiffer,
  ``Fast-classifying, high-accuracy spiking deep networks through weight and
  threshold balancing,'' in \emph{2015 International Joint Conference on Neural
  Networks (IJCNN)}, July 2015, pp. 1--8.

\bibitem{Rueckauer2017}
\BIBentryALTinterwordspacing
B.~Rueckauer, I.-A. Lungu, Y.~Hu, M.~Pfeiffer, and S.-C. Liu, ``Conversion of
  continuous-valued deep networks to efficient event-driven networks for image
  classification,'' \emph{Frontiers in Neuroscience}, vol.~11, p. 682, 2017.
  [Online]. Available:
  \url{https://www.frontiersin.org/article/10.3389/fnins.2017.00682}
\BIBentrySTDinterwordspacing

\bibitem{Hunsberger15}
\BIBentryALTinterwordspacing
E.~Hunsberger and C.~Eliasmith, ``Spiking deep networks with {LIF} neurons,''
  \emph{CoRR}, vol. abs/1510.08829, 2015. [Online]. Available:
  \url{http://arxiv.org/abs/1510.08829}
\BIBentrySTDinterwordspacing

\bibitem{Liu17}
\BIBentryALTinterwordspacing
Q.~Liu, Y.~Chen, and S.~B. Furber, ``Noisy softplus: an activation function
  that enables snns to be trained as anns,'' \emph{CoRR}, vol. abs/1706.03609,
  2017. [Online]. Available: \url{http://arxiv.org/abs/1706.03609}
\BIBentrySTDinterwordspacing

\bibitem{Zambrano17}
\BIBentryALTinterwordspacing
D.~Zambrano, R.~Nusselder, H.~S. Scholte, and S.~M. Bohte, ``Efficient
  computation in adaptive artificial spiking neural networks,'' \emph{CoRR},
  vol. abs/1710.04838, 2017. [Online]. Available:
  \url{http://arxiv.org/abs/1710.04838}
\BIBentrySTDinterwordspacing

\bibitem{Rueckauer18}
B.~Rueckauer and S.~Liu, ``Conversion of analog to spiking neural networks
  using sparse temporal coding,'' in \emph{2018 IEEE International Symposium on
  Circuits and Systems (ISCAS)}, May 2018, pp. 1--5.

\bibitem{LeCun98}
Y.~LeCun, L.~Bottou, Y.~Bengio, and P.~Haffner, ``{Gradient-based learning
  applied to document recognition},'' \emph{Proceedings of the IEEE}, vol.~86,
  no.~11, pp. 2278--2324, 1998.

\bibitem{Miyashita16}
\BIBentryALTinterwordspacing
D.~Miyashita, E.~H. Lee, and B.~Murmann, ``Convolutional neural networks using
  logarithmic data representation,'' \emph{CoRR}, vol. abs/1603.01025, 2016.
  [Online]. Available: \url{http://arxiv.org/abs/1603.01025}
\BIBentrySTDinterwordspacing

\bibitem{Gerstner2002}
W.~Gerstner and W.~M. Kistler, \emph{Spiking Neuron Models: Single Neurons,
  Populations, Plasticity}.\hskip 1em plus 0.5em minus 0.4em\relax Cambridge
  University Press, 2002.

\bibitem{tensorflow2015}
\BIBentryALTinterwordspacing
M.~Abadi, A.~Agarwal, P.~Barham \emph{et~al.}, ``{TensorFlow}: Large-scale
  machine learning on heterogeneous systems,'' 2015, software available from
  tensorflow.org. [Online]. Available: \url{https://www.tensorflow.org/}
\BIBentrySTDinterwordspacing

\bibitem{TensorFlow2018}
TensorFlow, ``A guide to tf layers: Building a convolutional neural network,''
  https://tensorflow.google.cn/tutorials/layers, 2018, accessed: 2018-04-10.

\end{thebibliography}
\end{document}